\documentclass[twoside]{article}

\usepackage[preprint]{aistats2027}
\usepackage{graphicx}
\usepackage{booktabs}
\usepackage{subcaption}
\usepackage{notations}  
\usepackage{natbib}
\usepackage{placeins}
\usepackage{hyperref}
\usepackage[dvipsnames]{xcolor}
\usepackage[table]{xcolor}
\usepackage{wasysym}   
\usepackage{float}   
\usepackage{makecell}
\usepackage[normalem]{ulem}
\usepackage{comment}
\usepackage{thm-restate}
\usepackage{tabularx}
\usepackage{wrapfig}
\usepackage{paralist}
\usepackage{multirow}
\usepackage{makecell}
\usepackage[capitalize,noabbrev]{cleveref}

\theoremstyle{plain}
\newtheorem{theorem}{Theorem}[section]

\newtheorem{lemma}[theorem]{Lemma}

\theoremstyle{definition}

\newtheorem{assumption}[theorem]{Assumption}

\theoremstyle{remark}

\crefname{example}{Example}{Examples}
\crefname{theorem}{Theorem}{Theorems}
\crefname{lemma}{Lemma}{Lemmas}
\Crefname{lemma}{Lemma}{Lemmas}
\crefname{remark}{Remark}{Remarks}
\crefname{proposition}{Proposition}{Propositions} 
\crefname{corollary}{Corollary}{Corollaries}
\crefname{definition}{Definition}{Definitions}
\crefname{assumption}{Assumption}{Assumptions}

\begin{document}

\twocolumn[

\runningtitle{Causal Discovery for Linear Models with Latent Confounding via Identifiable Representation Learning}

\aistatstitle{Causal Discovery for Linear Models with Latent Confounding \\via Multi-domain Identifiable Representation Learning}

\aistatsauthor{ Danru Xu \And Aapo Hyvärinen \And  Sara Magliacane}

\aistatsaddress{ University of Amsterdam \And  University of Helsinki \And Saarland University\\ University of Amsterdam} ]

\begin{abstract}
  Causal discovery with latent confounding is particularly challenging, but a common setting in real-world data.  In this work, we focus on linear causal models with latent confounding, represented by acyclic directed mixed graphs (ADMGs). 
  Existing methods that fully identify ADMGs typically assume non-Gaussianity or graphical constraints, such as bow-freeness.
  In contrast, we propose Linear ADMG Learning (LinADMG), a framework for learning  ADMGs for linear causal models through identifiable representation learning. Rather than relying on non-Gaussianity or bow-freeness, LinADMG leverages multi-domain settings to identify the noise variables. We show that the identifiability of the noise variables is sufficient for recovering all direct causal relations, and that, under an additional latent-confounding no-cancellation condition, LinADMG can identify the whole ADMG. We implement LinADMG with four identifiable representation learning approaches from the literature and validate the methods on simulated and semisynthetic datasets.
\end{abstract}


\section{Introduction}
\label{sec:introduction}

\begin{figure*}[t]
    \centering
    \includegraphics[width=1.0\textwidth]{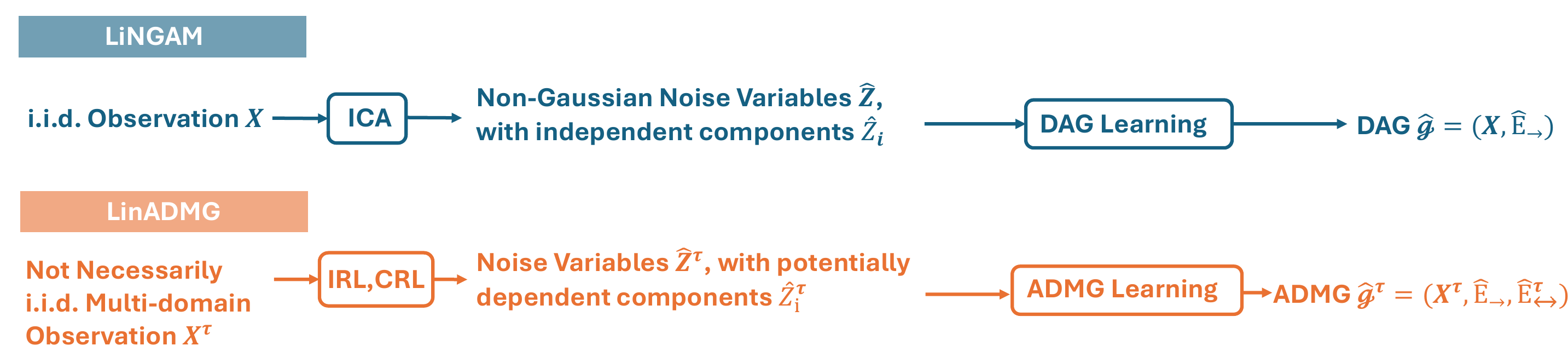}
    \caption{Overview of LiNGAM and LinADMG (our method) pipelines. LiNGAM applies ICA to the observed data $\Xb$ to estimate independent non-Gaussian noise variable $\widehat{\Zb}$, from which it infers the directed edge set $\widehat{E}_{\rightarrow}$ of a DAG. LinADMG replaces ICA with any appropriate IRL/CRL on a multi-domain dataset $\Xb^{\tau}$ with domain index $\tau$ to estimate dependent noise variable $\widehat{\Zb}^{\tau}$ without specific parametric assumptions on the noise variables, and thereby enabling the additional recovery of bidirected edges $\widehat{E}^{\tau}_{\leftrightarrow}$ representing latent confounding.}
    \label{fig:linadmg_intro}
\end{figure*}

Causal discovery aims to recover causal relationships from observational data \citep{spirtes2013causal, pearl2003causality,peters2017elements}. However, a common challenge in real-world data is the existence of latent confounding: an unobserved common cause of observed variables. Classical constraint-based methods allowing latent confounding, most notably FCI \citep{spirtes2013causal}, return only a Markov equivalence class rather than a unique causal structure.

To pursue a stronger identification goal for the causal structure with latent confounding, we first revisit the source-separation
perspective underlying LiNGAM \citep{shimizu2006linear}. In a linear structural causal model without latent confounding, the observed variables can be regarded as linear mixtures of the noise variables. By assuming mutually independent and non-Gaussian noise variables, independent component analysis (ICA) \citep{Hyva00NN} can be used to recover the unobserved noise variables and the corresponding mixing transformation up to permutation and component-wise scaling. The remaining ambiguities can be resolved using the acyclicity property of the causal graph, thereby fully recovering the underlying directed acyclic graph (DAG).

Latent confounding, however, fundamentally changes this picture. Suppose the underlying causal model is a DAG over both observed and latent variables. Marginalizing the latent variables generally destroys the DAG representation over the observed variables alone: latent common causes induce confounding relationships in addition to direct causal effects. Acyclic directed mixed graphs (ADMGs) provide a way to represent this structure, with directed edges encoding direct causal relations and bidirected edges encoding latent confounding \citep{richardson2003markov}. At the same time, marginalization can induce statistical dependence among the noise terms associated with different observed variables, violating the independent-source assumption underlying standard ICA. As a result, learning an ADMG is substantially more challenging than learning a DAG.

Existing extensions of LiNGAM to latent-variable settings exploit non-Gaussian noise variables \citep{hoyer2008estimation,salehkaleybar2020learning,cai2023causal,schkoda2024causal}, while \citet{wang2023causal} establish identification of the full ADMG under additional structural assumptions, including bow-freeness.

In this work, we also aim to fully identify the ADMG, but from a different perspective. As shown in Fig.~\ref{fig:linadmg_intro}, inspired by LiNGAM, we propose a framework, LinADMG, that follows the same principle: first identify the latent noises and then use the acyclicity property of ADMGs to recover the causal graph. However, unlike LiNGAM, latent confounding makes the noises dependent, so ICA is no longer suitable.

Recent developments in identifiable representation learning (IRL), including causal representation learning (CRL) \citep{scholkopf2021toward}, provide promising alternatives to ICA for LinADMG. IRL studies conditions under which latent representations can be identified from observed data generated through an unknown mixing process. CRL can be viewed as a special form of IRL, where the recovered latent variables are additionally equipped with causal
semantics.
A key insight from IRL is that latent representations that are not identifiable from a single observational distribution may become identifiable when we can access multiple domains in which the representation distributions vary. Such variation may arise from mixture modeling \citep{kivva2022identifiability}, sparse masking \citep{xu2024sparsity}, hard interventions \citep{varici2024linear}, or other changes in the latent distribution \citep{xu2026identifying}.

We adopt the same principle in LinADMG. Specifically, we consider settings in which the direct causal relations among the observed variables remain invariant, while the noise distribution varies across domains. We summarize our contributions as follows:

\begin{compactitem}
    \item 
    We propose LinADMG, a general framework that views linear causal discovery with latent confounding, which first identifies a potentially dependent noise representation and then recovers the causal graph. This extends the source-separation perspective underlying LiNGAM beyond independent non-Gaussian noises to potentially dependent noises identifiable through distributional variability. 
    \item 
    We show that noise variable identifiability is sufficient to recover all direct causal relations. Under an additional latent-confounding no-cancellation condition, which captures a different aspect of the model than bow-freeness and faithfulness, LinADMG further identifies the complete bidirected structure and therefore the full ADMG.
    \item 
    We implement LinADMG with four representative identifiable representation-learning settings: Diagonal Gaussian Mixture Models \citep{kivva2022identifiability}, Unpaired Partial Observations \citep{xu2024sparsity}, Multi-domain Scalings \citep{xu2026identifying}, and Unknown Hard Interventions \citep{varici2024linear}. These are illustrative examples, and LinADMG applies to any method that identifies the latent variables.
    We conduct experiments on synthetic and semisynthetic datasets and show accurate recovery under the linear model. 
\end{compactitem}

\section{Preliminaries}
\label{sec: prob setting}
We consider linear structural causal models (SCMs) with latent confounding and do not consider selection bias \citep{spirtes1995}. Let $\Xb=(X_1,\ldots,X_n)^T\in\mathbb{R}^n$ denote the observed variables and $\Ub\in\mathbb{R}^m$ the latent variables. Following \citet{hoyer2008estimation}, we assume the underlying causal model is
\begin{equation}
\begin{pmatrix}
\Xb\\
\Ub
\end{pmatrix}
=
\begin{pmatrix}
B & B_1\\
\mathbf{0} & B_2
\end{pmatrix}
\begin{pmatrix}
\Xb\\
\Ub
\end{pmatrix}
+
\begin{pmatrix}
\Eb_1\\
\Eb_2
\end{pmatrix},
\label{eq:dag-sem}
\end{equation}
where $B\in\mathbb{R}^{n\times n}$ encodes direct causal effects among the observed variables, $B_1\in\mathbb{R}^{n\times m}$ encodes effects from the
latent to the observed variables, and $B_2\in\mathbb{R}^{m\times m}$ encodes
causal relations among the latent variables. We assume that all components in $\Eb_1$ and $\Eb_2$ are mutually independent. The zero lower-left block excludes causal effects from observed variables to latent variables, so the latent variables are not descendants of the observed variables.

Solving the structural equations for $\Ub$ and substituting them into the equations for $\Xb$ gives
\begin{equation}
\Xb=B\Xb+\Zb,
\qquad
\Zb:=B_1(I-B_2)^{-1}\Eb_2+\Eb_1.
\label{eq:base-sem}
\end{equation}
Although $\Eb_1$ and $\Eb_2$ are mutually independent, marginalizing the latent variables can potentially induce dependence among the components of the composite noise $\Zb:=\{Z_1,...,Z_n\}$.

We use an acyclic directed mixed graph (ADMG) \citep{richardson2003markov} to represent the causal structure among the observed variables induced by the SCM in Eq.~\ref{eq:dag-sem}. With a slight abuse of notation, we use $\Xb$ to denote the corresponding node set:
\begin{equation}
\label{equ:admg}
\Gcal=(\Xb,E_{\to},E_{\leftrightarrow}).
\end{equation}
A directed edge $X_i\to X_j\in E_{\to}$ is present when $B_{ji}\neq 0$, representing a direct causal relation from $X_i$ to $X_j$. We assume the directed subgraph $(\Xb,E_{\to})$ is acyclic. A bidirected edge $X_i\leftrightarrow X_j\in E_{\leftrightarrow}$ is present when there exists a latent variable $U_i\in\Ub$ that is a common ancestor of $X_i$ and $X_j$ through directed paths whose intermediate nodes are all latent. Thus, $E_{\to}$ encodes the direct causal relation within $\Xb$, while $E_{\leftrightarrow}$ captures the latent-confounding induced by $\Ub$. An ADMG is \emph{bow-free} if no pair of observed variables is connected by both a directed edge and a bidirected edge.

Our goal is to recover the ADMG
$\Gcal=(\Xb,E_{\rightarrow},E_{\leftrightarrow})$, including both the directed causal structure $E_{\rightarrow}$ and the latent-confounding structure
$E_{\leftrightarrow}$. Importantly, unlike latent variable LiNGAM (lvLiNGAM) \citep{hoyer2008estimation} and its variants \citep{salehkaleybar2020learning,cai2023causal,schkoda2024causal}, our work does \emph{not} necessarily require the exogenous variables $\Eb_1$ and $\Eb_2$ to be non-Gaussian. Moreover, while the Markov property follows from our SCM assumptions,
we do not require faithfulness, which assumes that every conditional independence in the observed distribution corresponds to an $m$-separation in the ADMG \citep{richardson2003markov, spirtes2001causation}.
Instead, we require that the noise variables $\Zb$ in Eq.~\ref{eq:base-sem} can be \emph{identified} by linear representation $W\Xb$ up to \emph{permutation and component-wise scaling} \citep{comon1994independent}, i.e., there exists a permutation matrix $P$ and an invertible diagonal matrix $D$ such
that $W\Xb=DP\Zb$, where $W\in \RR^{n\times n}$.

In general, however, even independent $\Zb$ are not identifiable from the observed variables $\Xb$
without additional assumptions 
\citep{comon1994independent,Hyva99NN,
locatello2019challenging}.
Moreover, if $\Zb$ are potentially dependent, this is even more challenging. To solve this, a common strategy from IRL/CRL is to use distributional variability/heterogeneity in the latent variables across domains. 
We adopt the same principle here. Specifically, we define $\Zb^{\tau}$ as the noise variables in domain $\tau\in[T]$ and allow the noise variables to follow various distributions $P^{\tau}_{\Zb}$. 
We assume the structural coefficient matrix $B$ is invariant. Then we have the domain-specific observation variable $\Xb^{\tau}$:
\begin{equation}
\label{equ: I-B}
\Xb^{\tau}
=
(I-B)^{-1}\Zb^{\tau},
\qquad
\Zb^{\tau}\sim P^{\tau}_{\Zb},
\qquad
\tau\in[T].
\end{equation}
In practice, such domains can correspond to different environments, experimental conditions, populations, or interventions that change the distribution of noises. 

Next, we present four representative identifiable settings adopted from prior IRL/CRL works as \emph{examples}:

\textbf{Multi-domain Scalings (MD-Scalings)} \citep{xu2026identifying}.
Let $\Zb$ follow any arbitrary distribution with finite second moment and define $\Zb^\tau=D^\tau\Zb$,
where $D^\tau=\operatorname{diag}(d^\tau_{1},..., d^\tau_{n})$ is domain-specific. Identifiability of $\Zb^{\tau}$ up to permutation and scaling can be achieved with sufficient variability of the scaling matrix $D^\tau$, formalized by the full-rank condition on the matrix whose rows contain all fourth-order monomials $\{d^\tau_{i_1}d^\tau_{i_2}d^\tau_{i_3}d^\tau_{i_4}\}_{i_1\leq i_2\leq i_3\leq i_4}$.

\textbf{Diagonal Gaussian Mixture Models (DiagGMMs)}\citep{kivva2022identifiability}.
Let $u^\tau$ denote the domain label. We generate
$\Zb^\tau\mid u^\tau
\sim
\mathcal{N}\!\left(
    \boldsymbol{\mu}(u^\tau),\,\Sigmab(u^\tau)
\right)$,
where $\Sigmab$ is a diagonal. The domain label is unobserved, so the samples from all domains are mixed, resulting in a Gaussian mixture distribution. Identifiability up to permutation and scaling holds when there exists a pair of domains $\tau\neq\tau'$ such that the variance ratios
$\operatorname{Var}(\Zb_i^\tau)/\operatorname{Var}(\Zb_i^{\tau'})$
are distinct across $i$.

\textbf{Unpaired Partial Observations (UP-Observations)} \citep{xu2024sparsity}.
Let $\Zb$ follow an arbitrary distribution and define $\Zb^\tau=M^\tau \Zb$, $M^\tau=\operatorname{diag}(m_1^\tau,\ldots,m_n^\tau)$, where $m_i^\tau\in\{0,1\}$ refers whether the $i$th component is observed. Thus, each sample uses only a subset of the latent variables, and the resulting representation $\Zb^\tau$ is sparse. The masking pattern $\tau$ is unobserved and the samples are unpaired/mixed across patterns. Let $S^\tau:=\{i:m_i^\tau=1\}$ denote the active support under pattern $\tau$. Identifying $\Zb^\tau$ up to permutation and scaling holds under the sufficient support-variability condition: $\forall i\in[n]$,
$\bigcup_{\tau\in[T]:i\notin S^{\tau}} S^{\tau} = [n]\setminus\{i\}$.

\textbf{Unknown Multi-node Hard Interventions (UMN-Hard Interventions)} \citep{varici2024linear}. Assume that the noise variables $\Zb$ follow an additive-noise model. In each interventional domain $\tau$, an unknown subset $I^\tau\subseteq[n]$ is targeted. A hard intervention on $i\in I^\tau$ removes the dependence of $Z_i$ on its parents and replaces its observational conditional distribution $p_i(z_i\mid z_{\mathrm{pa}(i)})$ with an interventional distribution $q_i^\tau(z_i)$ that is independent of its parents. Define the intervention-signature matrix $D_{\mathrm{int}}\in\{0,1\}^{n\times T}$ by
$[D_{\mathrm{int}}]_{i,\tau}=\mathbf{1}\{i\in I^\tau\}$.
The intervention targets are unknown, but the collection of interventions must be sufficiently diverse: $D_{\mathrm{int}}$ must have full row rank. Under this condition, $\Zb^{\tau}$ are identifiable up to permutation and scaling.

Importantly, our general framework requires no particular form of heterogeneity of $P^{\tau}_{\Zb}$. Any data-generating setting in which the noise variables are identifiable up to permutation and scaling, while keeping $B$ invariant across domains, is compatible with our framework. 
The latent-confounding structure may itself vary across domains when the mechanism generating $P^\tau_{\Zb}$ changes. For example, under UMN-Hard Interventions setting, interventions may change the causal structure among latent variables and then change the latent-confounding relations induced among the observed variables. Thus, we allow the bidirected edge set to be domain-specific and denote it by $E^\tau_{\leftrightarrow}$. Finally, whether the domain index $\tau$ is observed depends on the identification setting, and our framework does not generally require access to domain labels.

\section{Linear ADMG learning}
\label{sec:linadmg}

We now describe a general framework called Linear ADMG Learning (LinADMG). It consists of three main steps. First, using observations $\{\Xb^\tau\}_{\tau=1}^T$ generated according to $\Xb^\tau=A\Zb^\tau$ with the shared mixing matrix $A:=(I-B)^{-1}$, we apply an IRL/CRL method to estimate $A$ and the corresponding domain-specific noise variables $\Zb^\tau$. Second, following the postprocessing procedure of LiNGAM \citep{shimizu2006linear}, we use acyclicity to resolve the permutation and scaling ambiguities of the estimated mixing matrix and then recover the directed structural coefficients $B$. Third, we estimate the bidirected edges by testing dependence among the estimated noise variables.

\textbf{Step 1: Estimate the mixing matrix $A$.}
We apply a linear representation-learning method to the observed data across domains. When the domain labels are observed, the input is the collection $\{\Xb^\tau:\tau\in[T]\}$; when they are unobserved, samples from all domains
are pooled and treated as a mixture distribution. We assume that the representation learning method used is able to identify $\Zb^{\tau}$, which equivalently means identifying $A$ up to permutation and component-wise scaling, i.e.,
$\widehat A
=
A P D,
\label{eq:crl-ambiguity}
$
for some permutation matrix $P$ and invertible diagonal matrix $D$. 

\textbf{Steps 2: Recover the directed edge coefficients $B$ via LiNGAM postprocessing \citep{shimizu2006linear}.}
Since $B$ is a coefficient matrix of a DAG, we can recover it from $\widehat A$ using the LiNGAM
post-processing procedure \citep{shimizu2006linear}. 
Inverting $\widehat A$ gives $\widehat A^{-1}
= D^{-1}P^T A^{-1}=
D^{-1}P^T (I-B)$.
Using the fact that $B$ is acyclic, $P$ can be identified since there is a unique row permutation of $\widehat A^{-1}$ such that every diagonal entry is nonzero \citep{shimizu2006linear}. Once the permutation is chosen, the remaining scaling ambiguity $D$ can be removed because $I-B$ has a unit diagonal. $B$ is then estimated as  
\begin{equation}
\widehat B=I-PD\widehat A^{-1}.
\end{equation}

Accordingly, the estimated directed edge set is 
\begin{equation}
\widehat E_{\to}
:=\{X^{\tau}_j\to X^{\tau}_i:\widehat B_{ij}\neq 0\},
\label{eq:directed-edges}
\end{equation}
which is invariant across domain by assumption. Details of the post-processing procedure are in App.~\ref{app:lingam-postprocessing}.

\textbf{Step 3: Recover the bidirected edge set $E^{\tau}_{\leftrightarrow}$.}
Given $\widehat B$, we can estimate the noise variables. When the domain label $\tau$ is observed, this can be done separately for each domain:
\begin{equation}
\widehat\Zb^{\tau}
:=
(I-\widehat B)\Xb^{\tau}.
\label{eq:estimated-noise}
\end{equation}
For each pair of recovered noise variables, we add a bidirected edge if they are dependent in domain $\tau$: 
\begin{equation}
\widehat{\Zb}^{\tau}_i\not\!\perp\!\!\!\perp \widehat{\Zb}^{\tau}_j
\quad\Longrightarrow\quad
X^{\tau}_i\leftrightarrow X^{\tau}_j
\in
\widehat E^{\tau}_{\leftrightarrow}.
\label{eq:bidirected-edges}
\end{equation}
Empirically, we can test for dependence between each pair of recovered exogenous variables with an appropriate independence test given a significance $\alpha$.

We can also apply our method to settings where the domain index is unobserved, and we have a mixture of observations $\Xb$. In this case, we apply the test over $\widehat{\Zb}$.

\section{ADMG identification via LinADMG}
\label{sec: identi admg}
We next provide conditions under which the proposed framework, LinADMG, identifies the underlying ADMG. By identification of an ADMG, we mean exact recovery of both directed and bidirected edge sets. 

We proceed in two stages. First, in Lemma~\ref{lem:directed_identification}, we show that, without imposing any additional assumption, LinADMG can identify all directed edges but a subset of the bidirected edges, because in cases of fine-tuned parameters, the effect from latent confounding may be canceled and therefore cannot be detected by our method. The proof is in App.~\ref{app: proof_lemma_1}.

\begin{restatable}[Directed-edge identification and bidirected-edge subset recovery]{lemma}{directedIdentification}
\label{lem:directed_identification}
Consider the domain-specific ADMG $\Gcal^\tau = (\Xb^\tau,E_{\rightarrow},E^\tau_{\leftrightarrow})$
generated by the linear SCM in Sec.~\ref{sec: prob setting}. Following the Steps in LinADMG, we have $\widehat B=B$,
and the directed edge set is identified and no false-positive estimated bidirected edges, i.e., 
$\widehat E_{\rightarrow}=E_{\rightarrow}$, $\widehat E^\tau_{\leftrightarrow}
\subseteq
E^\tau_{\leftrightarrow}$.

\end{restatable}

To identify all the bidirected edges, we introduce a generic no-cancellation condition ensuring that latent confounding is detectable in the noise dependencies.

\begin{assumption}[Latent-confounding no-cancellation]
\label{assump: no-cancel}
For every domain $\tau$ and every pair of distinct noise variables $Z^{\tau}_i,Z^{\tau}_j$, $i\neq j$, 
\begin{equation}
X_i^{\tau}\leftrightarrow X_j^{\tau}\in E^\tau_{\leftrightarrow}
\quad\Longrightarrow\quad
Z_i^\tau\not\!\perp\!\!\!\perp Z_j^\tau.
\end{equation}

\end{assumption}

Asm.~\ref{assump: no-cancel} connects our latent confounding definition of bidirected edges to the error-dependence parameterization commonly used in linear structural equation models (SEM) with unobserved confounding. In correlated-error SEMs, \citet{bhattacharya2021differentiable} define a bidirected edge
$X_i\leftrightarrow X_j$ whenever the corresponding error covariance is nonzero, while \citet{wang2023causal} use bidirected edges to encode dependence between the corresponding error terms. In these formulations, the relation between a bidirected edge and dependence among the noise variables is therefore built directly into the model. 

Different from their definition, we define $E^\tau_{\leftrightarrow}$ through the latent projection of the underlying causal DAG: $X_i^{\tau}\leftrightarrow X_j^{\tau}$ represents latent confounding between the two observed variables. After the latent variables are marginalized, however, the latent effects can cancel at the level of the noise dependence, as shown in Fig~\ref{fig:cancel_assump}. (b), where the dependence between $Z_1^\tau$ and $Z_2^\tau$ vanishes. Asm.~\ref{assump: no-cancel} rules out such cancellations, and thus aligns the definition of bidirected edges used in our formulation with the noise-dependence parameterization used in prior work.

Previous work commonly assumes bow-freeness with faithfulness because disentangling the causal effect from directed causal relations and a common latent confounder is challenging. We provide two examples in Fig.~\ref{fig:cancel_assump} showing that these assumptions, compared with Asm.~\ref{assump: no-cancel}, capture different aspects of the model and do not imply one another.

\begin{figure}[t]
    \centering
\includegraphics[width=0.48\textwidth]{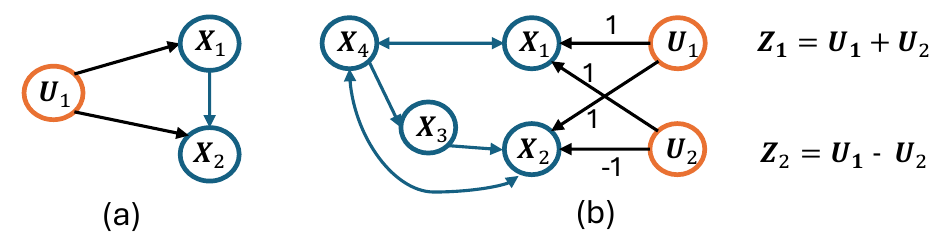}
\caption{Examples showing that Asm.~\ref{assump: no-cancel} and bow-freeness/faithfulness capture different properties.
(a) $X_1$ and $X_2$ share the latent parent $U_1$ while $X_1\to X_2$, resulting a violation of bow-freeness. Nevertheless, the latent confounding is detectable by the dependency of noise variables, so Asm.~\ref{assump: no-cancel} holds.
(b) Conversely, bow-freeness and faithfulness do not guarantee Asm.~\ref{assump: no-cancel}. Let $U_1$ and $U_2$ be independent Gaussian variables with equal variance. The resulting composite noise terms are $Z_1=U_1+U_2$ and $Z_2=U_1-U_2$, so
$\operatorname{Cov}(Z_1,Z_2)=0$ and hence $Z_1\perp\!\!\!\perp Z_2$, despite the existence of $X_1$ and $X_2$ as they share latent confounders.}
\label{fig:cancel_assump}
\end{figure}

In Thm.~\ref{thm:linadmg-recovery}, we show that under Asm.~\ref{assump: no-cancel}, LinADMG identifies the complete ADMG, with proof in App.~\ref{app: proof_thm}.

\begin{restatable}[Identifiability of LinADMG]{theorem}{linadMGrecovery}
\label{thm:linadmg-recovery}
Suppose that the conditions of Lemma~\ref{lem:directed_identification} hold and that Asm.~\ref{assump: no-cancel} is satisfied. Then, LinADMG identifies the complete domain-specific ADMG $\Gcal^\tau
=(\Xb^{\tau},E_{\rightarrow},E^\tau_{\leftrightarrow})$ 
for every domain $\tau$, i.e., $\widehat E_{\rightarrow}=E_{\rightarrow}$
and $\widehat E^\tau_{\leftrightarrow}= E^\tau_{\leftrightarrow}$.
\end{restatable}

\section{Implementation}
\label{sec: imple}

We implement LinADMG under four identifiable data-generating settings described in Sec.~\ref{sec: prob setting}. Each setting has a corresponding IRL/CRL method for Step 1, taking $\{\Xb^{\tau}\}_{\tau=1}^T$ as input and outputting $\widehat{A}$ or its inverse. We provide a broader discussion of IRL/CRL methods compatible with our framework in App.~\ref{app: IRL and CRL}.

\textbf{\texttt{LinADMG-MuDoNLL} with MD-Scalings.} 
Our implementation, following \citet{xu2026identifying}, first computes the empirical covariance of $\Xb^{\tau}$, denoted as $S^{\tau}$, and define the model-implied covariance as 
\begin{equation}
    \widehat{\Sigma}^{\tau}=A D^{\tau} (I-L)^{-1}
    \operatorname{diag}(\sigma^2)
    (I-L)^{-T} D^{\tau} A^T,
\end{equation}
where $\sigma$ and a zero-diagonal latent structural matrix $L$ is used to model the covariance structure of latent $\Zb$. 
While \citet{xu2026identifying} choose a covariance-matching objective, we instead optimize a Gaussian negative log-likelihood (NLL) since it is a scale-adaptive measure of covariance mismatch:
\begin{align}
    &\sum_{\tau=1}^{T}
    [
        \log\det\!\left(\widehat{\Sigma}^{\tau}+\epsilon I\right)
        +
        \operatorname{tr}\!\left(
            \left(\widehat{\Sigma}^{\tau}+\epsilon I\right)^{-1}S^{\tau}
        \right)
    ],
\end{align}
where $\epsilon$ is a small positive value for numerical stability. 

\textbf{\texttt{LinADMG-diagGMMs} with DiagGMM.}
\citet{kivva2022identifiability} proved that when we apply a Gaussian mixture estimator, identifiability can hold with a (piecewise-) linear decoder-encoder structure. Inspired by this, we implement the special case, linear model, by maximizing the likelihood of a $T$-component diagonal Gaussian mixture with weights $\pi^{\tau}$, latent means $\mu^{\tau}$, latent variances $v^{\tau}$, and a shared linear decoder $W$:
\begin{align}
    -\frac{1}{N}
    \sum_{i=1}^{N}
    &\log\!\left[
        \sum_{\tau=1}^{T}
        \pi^{\tau}\,
        \mathcal{N}\!\left(
            x_i;\,
            W\mu^{\tau},\,
            W\operatorname{diag}(v^{\tau})W^T+\epsilon I
        \right)
    \right].
\end{align}
We initialize the parameters using $k$-means on observed data and optimize them with Adam. The final estimator of $W$ is the inverse of $\widehat{A}$.

\textbf{\texttt{LinADMG-SP} with UP-Observations.}
We follow the implementation of \citet{xu2024sparsity}, using a linear autoencoder with a sparsity constraint. The encoder $g$ and decoder $f$ are MLPs without nonlinear activations. After training, we extract the linear part $F$ of the affine decoder as $F_{:j}=f(\eb_j)-f(\mathbf{0})$ where $\eb_j$ is the $j$th standard basis vector. Since $F$ maps the recovered noises to observations, we return $\widehat A=F$.

\textbf{\texttt{LinADMG-UMNI} with UMN-Hard Interventions.}

We follow the implementation for the case of hard-intervention from \citet{varici2024linear}, computing Gaussian score differences from sample covariances of each domain. The method recovers a linear encoder $\widehat H$ which maps observations to latent variables. We return $\widehat A=\widehat H^{-1}$ with Moore-Penrose pseudoinverse.
\section{Related Work}
Classical constraint-based causal discovery methods, most notably FCI \citep{spirtes2013causal}, handle latent confounding by using conditional-independence tests. While these methods do not have parametric assumptions, and can be used with non-parametric tests, they return a Markov equivalence class containing multiple causal graphs, rather than a unique underlying graph. 

A different line of work relies on additional functional, distributional, or graphical assumptions to obtain stronger identification than is possible from conditional independences alone. Focusing on linear models, prior work can be roughly divided into two groups. 

The first group develops differentiable or score-based methods for learning ADMGs through continuous constrained optimization. \citet{bhattacharya2021differentiable} consider linear Gaussian structural equation models with correlated errors and derive differentiable algebraic constraints for several restricted ADMG classes: ancestral graphs, which exclude a bidirected edge between any pair connected by a directed path; arid graphs, which exclude c-trees; and bow-free graphs, which exclude pairs connected simultaneously by a directed and a bidirected edge. Their consistency guarantee recovers the true graph up to equivalence with respect to the equality constraints implied by the observed distribution. More recently, \citet{morinishi2025differentiable} extend this approach by assuming multivariate generalized normal errors, allowing non-Gaussian information to further resolve edge orientations, and provide almost-everywhere parameter identifiability for bow-free ADMGs. However, these approaches either recover the graph only up to an observational equivalence class or rely on bow-freeness and specific error-distribution assumptions. 

A second group of work leverages non-Gaussianity or independent noise variables. This line of work is inspired by LiNGAM \citep{shimizu2006linear}, which assumes no latent confounding, but connects causal discovery with source separation: under linearity and independent non-Gaussian noises, recovering the ICA mixing matrix identifies the underlying DAG. This idea has been extended to time-series models \citep{hyvarinen2010estimation} and nonlinear causal models \citep{monti2020causal,reizinger2023jacobian,montagna2025identifiability}. Recently, \citet{heurtebise2026multiview} replace non-Gaussianity with variance diversity in a multi-view setting, while still assuming independent noises.

A few recent works extend LiNGAM to handle latent confounding by leveraging overcomplete independent component analysis (OICA), beginning with \citet{hoyer2008estimation}, which considers linear latent-variable models in which the latent variables and exogenous variables are mutually independent and non-Gaussian. \citet{salehkaleybar2020learning} extend this to latent confounding by identifying causal orders among observed variables, while
\citet{maeda2020rcd} additionally infer direct causal relations between unconfounded pairs and detect pairs sharing latent confounders. These results, however, do not in general result in exact recovery of the ADMG. Cumulant-based methods \citep{cai2023causal,schkoda2024causal} provide asymptotically correct recovery of causal structure under different specialized assumptions on the latent-variable model, such as at least three observed children per latent confounder and a local bound on the number of latent confounders, respectively. In addition, \citet{wang2023causal} provided full recovery of the graph under the bow-free assumption. In contrast, our method chooses a different perspective, identifiable representation learning, and does not require non-Gaussianity or the specific
graphical and latent-structural restrictions required by these approaches.

\section{Experimental Results}
\label{sec: experiment result}

We assess the performance of the proposed \texttt{LinADMG-} methods using synthetic data and then present an application to a semisynthetic dataset \texttt{causalAssembly}.

\begin{figure*}[t]
    \centering
    \includegraphics[width=1.0\textwidth]{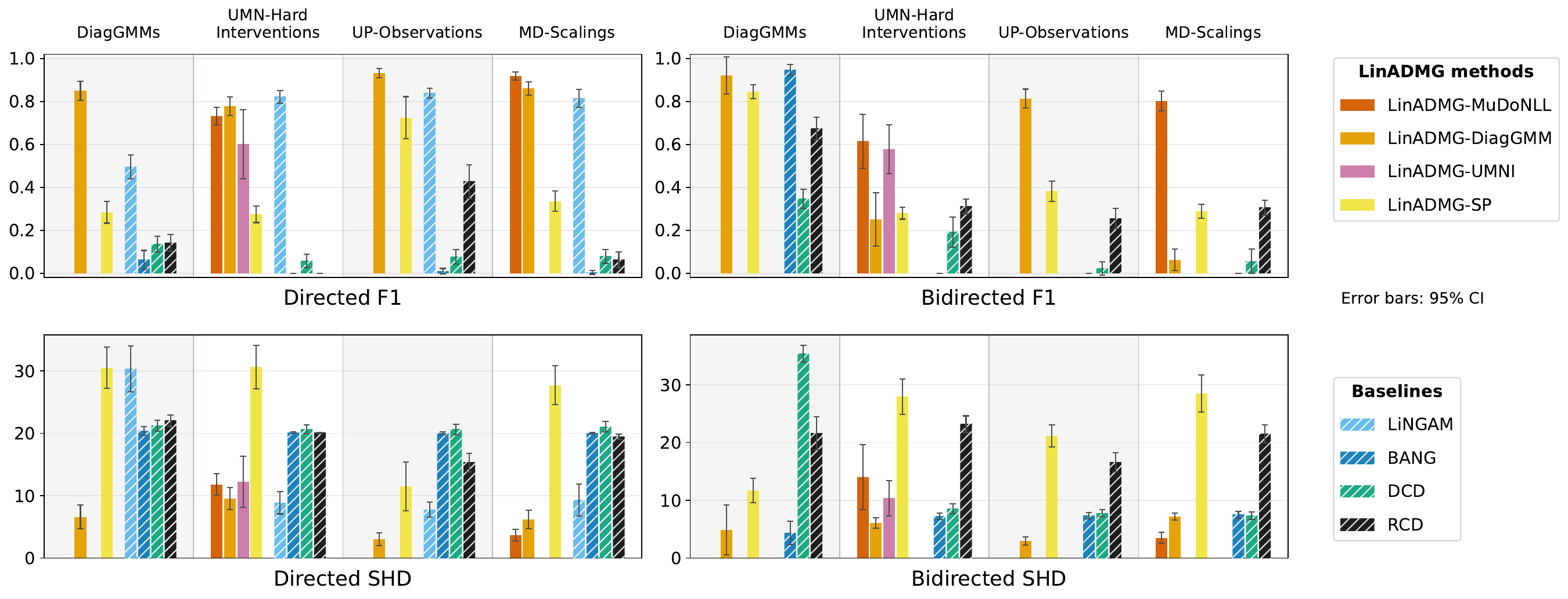}
    \caption{
        Estimation of the directed (left) and bidirected (right) edges of the ground-truth ADMG under four DGPs listed above. Bars report means over 20 random seeds, with  95\% confidence intervals. Warm colors are the proposed \texttt{LinADMG-} methods, while hatched cold colors are baselines. The top row reports F1 score ($\uparrow$), and the bottom row reports SHD ($\downarrow$). LiNGAM does not infer bidirected edges and is therefore omitted.
    }
    \label{fig:admg_f1_shd}
\end{figure*}

\textbf{Synthetic Data generation. }
We consider four linear data-generating processes (DGPs) described in Sec.~\ref{sec: prob setting}: MD-Scalings, UP-Observations, UMN-Hard Interventions, and Diag-GMMs. For the first three DGPs, we generate $m=\{5,10,20\}$ latent variables from a linear SCM, $\Ub=B_2\Ub+\Eb_2$, where the latent causal graph is sampled from an Erd\H{o}s-R\'enyi (ER) model with $k_U\in \{0, 1, 2, 3\}$, where each ER$-k_U$ graph contains $n\cdot k_U$ directed edges. The nonzero coefficients of $B_2$ are uniformly sampled from $[-0.5, -0.2]\cup [0.2, 0.5]$. We sample the scales of each component of $\Eb_2$ independently from $\mathrm{Unif}(0.2,0.5)$ and consider three noise distributions: Gaussian, Exponential, and Gumbel distributions. The three DGPs differ in how $\Ub$ generates $\Zb^\tau$ with domain number $T=\{1,2,3\}\times m$. 

For \emph{UMN-Hard Interventions}, we generate $T$ intervention domains and one additional observational reference domain. In each intervention domain $\tau$, we select a target set $\mathcal{I}^\tau$ and replace the structural equation of each targeted latent variable by an independent noise term. We sample target sets such that every latent variable is targeted in at least one domain and the binary domain-by-variable target matrix has full column rank. Specifically, we construct $B_2^\tau$ by setting the rows of $B_2$ corresponding to $\mathcal{I}^\tau$ to zero and generate $\Zb^\tau=(I-B_2^\tau)^{-1}\Eb_2$. 

For \emph{UP-Observations}, we generate a domain-specific binary mask: $M^\tau\in\{0,1\}^m$ with half active components. If a component is inactive in every masked domain, we activate it in a randomly selected domain. We additionally generate a reference domain with $M^\tau=I$, giving $T+1$ domains in total. Then apply it to $\Ub$ to get $\Zb^{\tau}=M^\tau\Ub$. 

For \emph{MD-Scalings}, we generate the domain-specific scaling matrix as
$D^\tau=\operatorname{diag}(d_1^\tau,\ldots,d_n^\tau)$, where the diagonal entries are sampled independently from $\mathrm{Unif}(0.2,1)$. We additionally generate an unscaled reference domain with $D^\tau=I$, giving $T+1$ domains in total. Then apply it to $\Ub$ to get $\Zb^{\tau}=D^\tau\Ub$. 

For \emph{Diag-GMMs}, we generate five domains indexed by an observed auxiliary variable $a^\tau\in\{-2,-1,0,1,2\}$. Conditional on $a^\tau$, we directly sample the composite noise from $\Zb^\tau\mid a^\tau \sim\mathcal{N}\bigl(\mu(a^\tau),\Sigma(a^\tau)\bigr)$, where $\mu(a)=\alpha+a\beta$ and $\Sigma(a)=(L_0+aL_1)(L_0+aL_1)^T+10^{-6}I$. $\alpha$ and $\beta$ are sampled independently from $\mathcal{N}(0,0.5^2)$. The matrices $L_0$ and $L_1$ are diagonal, and sampled independently from $\mathrm{Unif}(0.2,0.5)$ and $\mathcal{N}(0,0.08^2)$, respectively. Thus, though $\Zb^\tau$ are independent conditional on $a^{\tau}$, pooling the domains can induce dependence through the shared auxiliary variable. 

For all four DGPs, we consider $n=m$ observed variables $\Xb$. We sample the observed causal graph $\Dcal$ from an ER model with $k_X\in \{0, 1, 2, 3\}$, where each ER$-k_X$ graph contains $n \cdot k_X$ directed edges. Given $\Dcal$, we sample the linear edge weight $B$ uniformly from $[-0.5, -0.2]\cup [0.2, 0.5]$. For each domain $\tau$, we generate the observed variable $\Xb^\tau=(I-B)^{-1}\Zb^\tau$ with sample size $N\in\{1 000,5 000,50 000\}$. We run each setup over 20 random seeds.

\textbf{Baselines and Evaluation Metrics.} We 
compare our methods with DirectLiNGAM \citep{shimizu2006linear}, BANG \citep{wang2023causal}, DCD \citep{bhattacharya2021differentiable}, and RCD \citep{maeda2020rcd}. We pool observations generated via scaling and hard intervention DGPs since the baselines do not require the domain index. Due to computational constraints, we fit DCD using at most $10000$ pooled observations, while RCD uses at most $300$ observations. We evaluate directed and bidirected edges separately using F1 score and structural Hamming distance (SHD). Higher F1 and lower SHD mean better estimation.

\textbf{Synthetic Experiments Results.} Fig.~\ref{fig:admg_f1_shd} reports the accuracy with 95\% confidence intervals of directed- and bidirected-edge estimation across four DGPs with $n=m=10$, $ T=n$, $ k_X=2$, $k_U=0.5$, $N=50000$ for UMN-Hard Interventions dataset; otherwise, $N=5000$, and Gaussian noise. Overall, the proposed methods achieve the strongest identification, particularly \texttt{LinADMG-diagGMM} and \texttt{LinADMG-MuDoNLL} when their assumptions match the underlying domain-variation mechanism. Under UMN-Hard Interventions DGP, DirectLiNGAM remains competitive for directed-edge recovery. However, it does not support latent confounding estimation. We evaluate \texttt{LinADMG-UMNI} only under the UMN-Hard Interventions DGP because the UMNI representation-learning algorithm strictly requires hard interventions and a full-column-rank intervention design, which the other DGPs do not satisfy. Additional results in App.~\ref{app: all config} show that its reliable performance also requires a sufficiently large sample size within each domain. Finally, \texttt{LinADMG-SP} relies on structural sparsity and is therefore competitive under its matched dataset, but its recovery performance becomes worse when the domain variation does not produce the sparsity pattern assumed by the method.

In App.~\ref{app: B est}, we show near-perfect directed-edge weight estimation for \texttt{LinADMG-MuDoNLL} under MD-Scalings. In App.~\ref{app: all config}, we report the accuracy of directed- and bidirected-edge estimation across different configurations of $n$, sample numbers $N$, domain numbers $T$, graph density $k_X$, confounding density $k_U$, and noise distribution for all four DGPs. We also provide a comparison between the loss function proposed in Sec.~\ref{sec: imple} and those from the original papers, including comparing diagGMM (ours) with VaDE (used in \citet{kivva2022identifiability}) in App.~\ref{app: vade}, and comparing MuDoNLL (ours) with MuDoCov (proposed by \citet{xu2026identifying}) in App.~\ref{app: mudocov}, where our methods consistently outperform. 

We also compare our method with FCI \citep{spirtes1995} in App.~\ref{app: fci}. We first convert the estimated ADMG into PAG and then compute F1 for the skeleton, arrowhead, and tail. Our method significantly improves recovery of the correct tail and arrowhead, and shows comparable performance for skeleton accuracy. We also test our methods and baselines with bow-free underlying ADMG and non-Gaussian noise. The results in App.~\ref{app:bowfree} show that our methods maintain stable and robust performance, while the baselines show significant improvement due to the match of their methodological assumptions.

\textbf{\texttt{causalAssembly} with latent confounding. }
We use the causalAssembly benchmark \citep{gobler2024texttt}, a semi-synthetic benchmark derived from a real manufacturing assembly line, focusing on Stations~1 and~2. We first treat all $n_1=6$ Station~1 variables as latent and all $n_2=34$ Station~2 variables as observed. We generate 18 domains with 5000 samples each: one observational domain and 17 soft-intervention domains, each targeting 1-3 latent variables. We evaluate \texttt{LinADMG-diagGMM} and \texttt{LinADMG-MuDoNLL} over five optimization seeds. Both perform poorly in this setting, which we attribute partly to limited latent variability, since only one directed edge connects Station~1 to Station~2. We therefore also treat the first half of Station~2 variables in topological order as latent, leaving the remaining half observed. The resulting ADMG contains 25 directed and 54 bidirected edges. This results in a modest improvement. A major reason for the difficulty is the strong nonlinearity of the causalAssembly mechanisms, as \citet{gobler2024texttt} reports similarly poor performance of linear causal discovery methods on this benchmark. To separate nonlinear model misspecification from the latent-confounding challenge, we construct a linear version that preserves the graph topology but replaces the mechanisms with a linear SCM, with edge coefficients sampled uniformly from $[0.4,0.8]$. This results in substantially higher accuracy for both methods. We summarize the results in Tab.~\ref{tab:real-data-results}.

\begin{table}[t]
\centering
\footnotesize
\setlength{\tabcolsep}{2pt}
\renewcommand{\arraystretch}{0.95}
\caption{Average F1 scores $\pm$ standard deviation over 5 random seeds across different versions of causalAssembly with latent confounding with linear/nonlinear SCM and latent confounding settings.}
\label{tab:real-data-results}
\begin{tabular}{lllll}
\toprule
SCM & Latent & Method & Directed F1& Bidirected F1\\
\midrule

\multirow{2}{*}{Nonlinear}
& \multirow{2}{*}{$n_1$}
& \texttt{-diagGMM}
& $.020 \pm .000$
& $.004 \pm .000$ \\
&
&
\texttt{-MuDoNLL}
& $.074 \pm .000$
& $.004 \pm .000$ \\
\addlinespace[2pt]

\multirow{2}{*}{Nonlinear}
& \multirow{2}{*}{\makecell[l]{$n_1$ +\\50\% $n_2$}}
& \texttt{-diagGMM}
& $.000 \pm .000$
& $.554 \pm .000$ \\
&
&
\texttt{-MuDoNLL}
& $.085 \pm .000$
& $.576 \pm .000$ \\
\addlinespace[2pt]

\multirow{2}{*}{Linear}
& \multirow{2}{*}{\makecell[l]{$n_1$ +\\50\% $n_2$}}
& \texttt{-diagGMM}
& $.508 \pm .020$
& $.800 \pm .009$ \\
&
&
\texttt{-MuDoNLL}
& $.147 \pm .031$
& $.554 \pm .024$ \\

\bottomrule
\end{tabular}
\end{table}

\section{Conclusion and limitation}
\label{sec: discussion}
We introduced LinADMG, a general framework for learning linear ADMGs through identifiable representation learning. We show that LinADMG can recover the directed causal structure whenever the noise variables is identifiable up to permutation and scaling; it can further identify the full ADMG under a latent-confounding no-cancellation condition. Our results show how noise heterogeneity can provide an alternative source of identifiability for causal discovery with latent confounding. 

The current framework is limited to linear structural causal models. Extending the framework to nonlinear causal structural models is an important direction for future work. In addition, we currently consider only a selected set of identifiable representation-learning settings. Extending the framework to incorporate other sources of identifiability, such as temporal structure, could further broaden its applicability.

\newpage
\section*{Acknowledgements}
The research of DX and SM was supported by the Air Force Office of Scientific Research under award number FA8655-22-1-7155. Any opinions, findings, and conclusions or recommendations expressed in this material are those of the author(s) and do not necessarily reflect the views of the United States Air Force. We also thank SURF for the support in using the Dutch National Supercomputer Snellius.

\subsection*{AI use statement}

In this work, we used generative AI tools for generating synthetic data sets, polishing the theorem claim, assisting in the writing of proofs, implementing methods, cleaning and reformatting datasets, and interpreting results.
We have not used generative AI tools for helping develop theoretical models or conceptual frameworks, formulating mathematical claims, providing critical ingredients for proving mathematical claims, designing or providing feedback on research methodology or experiments, 
and proposing or refining hypotheses, assisting with translation, and supporting qualitative and thematic data analysis are not applicable to this work.
Additionally, we used generative AI tools for creating or modifying scientific figures or images, summarizing or analyzing existing literature, editing a research paper to improve readability, and identifying relevant literature.
We have reviewed all AI-assisted work.

We take responsibility for the final content of this work, including text, claims, or artifacts produced with the aid of generative AI.

\bibliography{reference}

@Article{Hyva00NN,
  author = 	 {A.~Hyv\"arinen and E.~Oja},
  title = 	 {Independent Component Analysis: Algorithms and
		    Applications}, 
  journal = 	 {Neural Networks},
  year = 	 {2000},
  volume = 	 {13},
  number = 	 {4-5},
  pages = 	 {411--430}
}

@Article{Hyva99NN,
  author = 	 {A.~Hyv\"arinen and P.~Pajunen},
  title = 	 {Nonlinear independent component analysis: Existence
		    and uniqueness results},
  journal = 	 {Neural Networks},
  year = 	 {1999},
  volume = 	 {12},
  number = 	 {3},
  pages = 	 {429--439}
}

@article{schkoda2024causal,
  title={Causal discovery of linear non-Gaussian causal models with unobserved confounding},
  author={Schkoda, Daniela and Robeva, Elina and Drton, Mathias},
  journal={arXiv preprint arXiv:2408.04907},
  year={2024}
}

@article{
morinishi2025differentiable,
title={Differentiable Causal Discovery of Linear Non-Gaussian Acyclic Models Under Unmeasured Confounding},
author={Yoshimitsu Morinishi and Shohei Shimizu},
journal={Transactions on Machine Learning Research},
issn={2835-8856},
year={2025}
}

@inproceedings{cai2023causal,
  title={Causal discovery with latent confounders based on higher-order cumulants},
  author={Cai, Ruichu and Huang, Zhiyi and Chen, Wei and Hao, Zhifeng and Zhang, Kun},
  booktitle={International conference on machine learning},
  pages={3380--3407},
  year={2023},
  organization={PMLR}
}

@inproceedings{bhattacharya2021differentiable,
  title={Differentiable causal discovery under unmeasured confounding},
  author={Bhattacharya, Rohit and Nagarajan, Tushar and Malinsky, Daniel and Shpitser, Ilya},
  booktitle={International Conference on Artificial Intelligence and Statistics},
  pages={2314--2322},
  year={2021},
  organization={PMLR}
}

@article{hoyer2008estimation,
  title={Estimation of causal effects using linear non-Gaussian causal models with hidden variables},
  author={Hoyer, Patrik O and Shimizu, Shohei and Kerminen, Antti J and Palviainen, Markus},
  journal={International Journal of Approximate Reasoning},
  volume={49},
  number={2},
  pages={362--378},
  year={2008},
  publisher={Elsevier}
}

@inproceedings{montagna2025identifiability,
  title={On the identifiability of causal graphs with multiple environments},
  author={Montagna, Francesco},
  booktitle={International Conference on Learning Representations},
  volume={2026},
  pages={13837--13868},
  year={2026}
}

@inproceedings{
xu2026identifying,
title={Identifying dependent components from multi-domain linear mixtures},
author={Danru Xu and Lauri Parkkonen and Sara Magliacane and Aapo Hyvarinen},
booktitle={Forty-third International Conference on Machine Learning},
year={2026}
}

@article{kivva2022identifiability,
  title={Identifiability of deep generative models without auxiliary information},
  author={Kivva, Bohdan and Rajendran, Goutham and Ravikumar, Pradeep and Aragam, Bryon},
  journal={Advances in Neural Information Processing Systems},
  volume={35},
  pages={15687--15701},
  year={2022}
}

@inproceedings{spirtes2013causal,
author = {Spirtes, Peter and Meek, Christopher and Richardson, Thomas},
title = {Causal inference in the presence of latent variables and selection bias},
year = {1995},
booktitle = {Proceedings of the Eleventh Conference on Uncertainty in Artificial Intelligence},
pages = {499–506},
numpages = {8}
}

@inproceedings{maeda2020rcd,
  title={RCD: Repetitive causal discovery of linear non-Gaussian acyclic models with latent confounders},
  author={Maeda, Takashi Nicholas and Shimizu, Shohei},
  booktitle={International Conference on artificial intelligence and statistics},
  pages={735--745},
  year={2020},
  organization={PMLR}
}

@article{shimizu2006linear,
  title={A linear non-Gaussian acyclic model for causal discovery.},
  author={Shimizu, Shohei and Hoyer, Patrik O and Hyv{\"a}rinen, Aapo and Kerminen, Antti and Jordan, Michael},
  journal={Journal of Machine Learning Research},
  volume={7},
  number={10},
  year={2006}
}

@article{salehkaleybar2020learning,
  title={Learning linear non-Gaussian causal models in the presence of latent variables},
  author={Salehkaleybar, Saber and Ghassami, AmirEmad and Kiyavash, Negar and Zhang, Kun},
  journal={Journal of Machine Learning Research},
  volume={21},
  number={39},
  pages={1--24},
  year={2020}
}

@article{wang2023causal,
  title={Causal discovery with unobserved confounding and non-Gaussian data},
  author={Wang, Y Samuel and Drton, Mathias},
  journal={Journal of Machine Learning Research},
  volume={24},
  number={271},
  pages={1--61},
  year={2023}
}

@InProceedings{xu2024sparsity,
  title = 	 {A Sparsity Principle for Partially Observable Causal Representation Learning},
  author =       {Xu, Danru and Yao, Dingling and Lachapelle, Sebastien and Taslakian, Perouz and Von K\"{u}gelgen, Julius and Locatello, Francesco and Magliacane, Sara},
  booktitle = 	 {Proceedings of the 41st International Conference on Machine Learning},
  pages = 	 {55389--55433},
  year = 	 {2024},
  volume = 	 {235}
}

@article{varici2024linear,
  title={Linear causal representation learning from unknown multi-node interventions},
  author={Var{\i}c{\i}, Burak and Acart{\"u}rk, Emre and Shanmugam, Karthikeyan and Tajer, Ali},
  journal={Advances in Neural Information Processing Systems},
  volume={37},
  pages={111614--111648},
  year={2024}
}

@article{hyvarinen2010estimation,
  title={Estimation of a structural vector autoregression model using non-Gaussianity.},
  author={Hyv{\"a}rinen, Aapo and Zhang, Kun and Shimizu, Shohei and Hoyer, Patrik O},
  journal={Journal of Machine Learning Research},
  volume={11},
  number={5},
  year={2010}
}

@inproceedings{monti2020causal,
  title={Causal discovery with general non-linear relationships using non-linear ICA},
  author={Monti, Ricardo Pio and Zhang, Kun and Hyv{\"a}rinen, Aapo},
  booktitle={Uncertainty in artificial intelligence},
  pages={186--195},
  year={2020},
  organization={PMLR}
}

@article{reizinger2023jacobian,
  title={Jacobian-based causal discovery with nonlinear ICA},
  author={Reizinger, Patrik and Sharma, Yash and Bethge, Matthias and Sch{\"o}lkopf, Bernhard and Husz{\'a}r, Ferenc and Brendel, Wieland},
  journal={Transactions on Machine Learning Research},
  year={2023}
}

@inproceedings{hyvarinen2017nonlinear,
  title={Nonlinear ICA of temporally dependent stationary sources},
  author={Hyvarinen, Aapo and Morioka, Hiroshi},
  booktitle={Artificial intelligence and statistics},
  pages={460--469},
  year={2017},
  organization={PMLR}
}

@inproceedings{hyvarinen2019nonlinear,
  title={Nonlinear ICA using auxiliary variables and generalized contrastive learning},
  author={Hyvarinen, Aapo and Sasaki, Hiroaki and Turner, Richard},
  booktitle={The 22nd international conference on artificial intelligence and statistics},
  pages={859--868},
  year={2019},
  organization={Pmlr}
}

@inproceedings{khemakhem2020variational,
  title={Variational autoencoders and nonlinear ica: A unifying framework},
  author={Khemakhem, Ilyes and Kingma, Diederik and Monti, Ricardo and Hyvarinen, Aapo},
  booktitle={International conference on artificial intelligence and statistics},
  pages={2207--2217},
  year={2020},
  organization={PMLR}
}

@article{brehmer2022weakly,
  title={Weakly supervised causal representation learning},
  author={Brehmer, Johann and De Haan, Pim and Lippe, Phillip and Cohen, Taco S},
  journal={Advances in Neural Information Processing Systems},
  volume={35},
  pages={38319--38331},
  year={2022}
}

@inproceedings{lippe2022citris,
  title={Citris: Causal identifiability from temporal intervened sequences},
  author={Lippe, Phillip and Magliacane, Sara and L{\"o}we, Sindy and Asano, Yuki M and Cohen, Taco and Gavves, Stratis},
  booktitle={International Conference on Machine Learning},
  pages={13557--13603},
  year={2022},
  organization={PMLR}
}

@inproceedings{lippe2023biscuit,
  title={Biscuit: Causal representation learning from binary interactions},
  author={Lippe, Phillip and Magliacane, Sara and L{\"o}we, Sindy and Asano, Yuki M and Cohen, Taco and Gavves, Efstratios},
  booktitle={Uncertainty in Artificial Intelligence},
  pages={1263--1273},
  year={2023},
  organization={PMLR}
}

@inproceedings{ahuja2023interventional,
  title={Interventional causal representation learning},
  author={Ahuja, Kartik and Mahajan, Divyat and Wang, Yixin and Bengio, Yoshua},
  booktitle={International conference on machine learning},
  pages={372--407},
  year={2023},
  organization={PMLR}
}

@article{von2023nonparametric,
  title={Nonparametric identifiability of causal representations from unknown interventions},
  author={von K{\"u}gelgen, Julius and Besserve, Michel and Wendong, Liang and Gresele, Luigi and Keki{\'c}, Armin and Bareinboim, Elias and Blei, David and Sch{\"o}lkopf, Bernhard},
  journal={Advances in Neural Information Processing Systems},
  volume={36},
  pages={48603--48638},
  year={2023}
}

@article{zhang2023identifiability,
  title={Identifiability guarantees for causal disentanglement from soft interventions},
  author={Zhang, Jiaqi and Greenewald, Kristjan and Squires, Chandler and Srivastava, Akash and Shanmugam, Karthikeyan and Uhler, Caroline},
  journal={Advances in Neural Information Processing Systems},
  volume={36},
  pages={50254--50292},
  year={2023}
}

@inproceedings{ahuja2024multi,
  title={Multi-domain causal representation learning via weak distributional invariances},
  author={Ahuja, Kartik and Mansouri, Amin and Wang, Yixin},
  booktitle={International Conference on Artificial Intelligence and Statistics},
  pages={865--873},
  year={2024},
  organization={PMLR}
}

@article{buchholz2023learning,
  title={Learning linear causal representations from interventions under general nonlinear mixing},
  author={Buchholz, Simon and Rajendran, Goutham and Rosenfeld, Elan and Aragam, Bryon and Sch{\"o}lkopf, Bernhard and Ravikumar, Pradeep},
  journal={Advances in Neural Information Processing Systems},
  volume={36},
  pages={45419--45462},
  year={2023}
}

@article{yao2023multi,
  title={Multi-View Causal Representation Learning with Partial Observability},
  author={Yao, Dingling and Xu, Danru and Lachapelle, S{\'e}bastien and Magliacane, Sara and Taslakian, Perouz and Martius, Georg and von K{\"u}gelgen, Julius and Locatello, Francesco},
  journal={International Conference on Learning Representations},
  year={2024}
}

@article{lachapelle2024nonparametric,
  title={Nonparametric Partial Disentanglement via Mechanism Sparsity: Sparse Actions, Interventions and Sparse Temporal Dependencies},
  author={Lachapelle, S{\'e}bastien and L{\'o}pez, Pau Rodr{\'\i}guez and Sharma, Yash and Everett, Katie and Le Priol, R{\'e}mi and Lacoste, Alexandre and Lacoste-Julien, Simon},
  journal={Journal of Machine Learning Research},
  volume={27},
  number={71},
  pages={1--90},
  year={2026}
}

@article{kong2023identification,
  title={Identification of nonlinear latent hierarchical models},
  author={Kong, Lingjing and Huang, Biwei and Xie, Feng and Xing, Eric and Chi, Yuejie and Zhang, Kun},
  journal={Advances in Neural Information Processing Systems},
  volume={36},
  pages={2010--2032},
  year={2023}
}

@article{kong2024learning,
  title={Learning discrete concepts in latent hierarchical models},
  author={Kong, Lingjing and Chen, Guangyi and Huang, Biwei and Xing, Eric and Chi, Yuejie and Zhang, Kun},
  journal={Advances in Neural Information Processing Systems},
  volume={37},
  pages={36938--36975},
  year={2024}
}

@article{dong2024on,
  title={On the parameter identifiability of partially observed linear causal models},
  author={Dong, Xinshuai and Ng, Ignavier and Huang, Biwei and Sun, Yuewen and Jin, Songyao and Legaspi, Roberto and Spirtes, Peter and Zhang, Kun},
  journal={Advances in Neural Information Processing Systems},
  volume={37},
  pages={30740--30771},
  year={2024}
}

@article{song2024causal,
  title={Causal temporal representation learning with nonstationary sparse transition},
  author={Song, Xiangchen and Li, Zijian and Chen, Guangyi and Zheng, Yujia and Fan, Yewen and Dong, Xinshuai and Zhang, Kun},
  journal={Advances in Neural Information Processing Systems},
  volume={37},
  pages={77098--77131},
  year={2024}
}

@InProceedings{morioka2023causal,
  title = 	 {Causal Representation Learning Made Identifiable by Grouping of Observational Variables},
  author =       {Morioka, Hiroshi and Hyvarinen, Aapo},
  booktitle = 	 {Proceedings of the 41st International Conference on Machine Learning},
  pages = 	 {36249--36293},
  year = 	 {2024},
  volume = 	 {235},
  publisher =    {PMLR},
  
}

@inproceedings{
li2025towards,
title={Towards Identifiability of Hierarchical Temporal Causal Representation Learning},
author={Zijian Li and Minghao Fu and Junxian Huang and Yifan Shen and Ruichu Cai and Yuewen Sun and Guangyi Chen and Kun Zhang},
booktitle={The Thirty-ninth Annual Conference on Neural Information Processing Systems},
year={2025}
}

@inproceedings{lachapelle2022disentanglement,
  title={Disentanglement via mechanism sparsity regularization: A new principle for nonlinear {ICA}},
  author={Lachapelle, S{\'e}bastien and Rodriguez, Pau and Sharma, Yash and Everett, Katie E and Le Priol, R{\'e}mi and Lacoste, Alexandre and Lacoste-Julien, Simon},
  booktitle={Conference on Causal Learning and Reasoning},
  pages={428--484},
  year={2022},
  organization={PMLR}
}

@InProceedings{zhang2024causal,
  title = 	 {Causal Representation Learning from Multiple Distributions: A General Setting},
  author =       {Zhang, Kun and Xie, Shaoan and Ng, Ignavier and Zheng, Yujia},
  booktitle = 	 {Proceedings of the 41st International Conference on Machine Learning},
  pages = 	 {60057--60075},
  year = 	 {2024},
  volume = 	 {235},
  publisher =    {PMLR}
}

@inproceedings{spirtes1995,
author = {Spirtes, Peter and Meek, Christopher and Richardson, Thomas},
title = {Causal inference in the presence of latent variables and selection bias},
year = {1995},
booktitle = {Proceedings of the Eleventh Conference on Uncertainty in Artificial Intelligence},
pages = {499–506},
numpages = {8},
series = {UAI'95}
}

@article{richardson2003markov,
  title={Markov properties for acyclic directed mixed graphs},
  author={Richardson, Thomas},
  journal={Scandinavian Journal of Statistics},
  volume={30},
  number={1},
  pages={145--157},
  year={2003},
  publisher={Wiley Online Library}
}

@article{comon1994independent,
  title={Independent component analysis, a new concept?},
  author={Comon, Pierre},
  journal={Signal processing},
  volume={36},
  number={3},
  pages={287--314},
  year={1994},
  publisher={Elsevier}
}

@inproceedings{locatello2019challenging,
  title={Challenging common assumptions in the unsupervised learning of disentangled representations},
  author={Locatello, Francesco and Bauer, Stefan and Lucic, Mario and Raetsch, Gunnar and Gelly, Sylvain and Sch{\"o}lkopf, Bernhard and Bachem, Olivier},
  booktitle={international conference on machine learning},
  pages={4114--4124},
  year={2019},
  organization={PMLR}
}

@article{hyvarinen2004blind,
  title={Blind separation of sources that have spatiotemporal variance dependencies},
  author={Hyv{\"a}rinen, Aapo and Hurri, Jarmo},
  journal={Signal processing},
  volume={84},
  number={2},
  pages={247--254},
  year={2004},
  publisher={Elsevier}
}

@inproceedings{zhang2011general,
  title={A general linear non-gaussian state-space model: Identifiability, identification, and applications},
  author={Zhang, Kun and Hyv{\"a}rinen, Aapo},
  booktitle={Asian Conference on Machine Learning},
  pages={113--128},
  year={2011},
  organization={PMLR}
}

@article{gomez2008measuring,
  title={Measuring directional coupling between EEG sources},
  author={G{\'o}mez-Herrero, Germ{\'a}n and Atienza, Mercedes and Egiazarian, Karen and Cantero, Jose L},
  journal={Neuroimage},
  volume={43},
  number={3},
  pages={497--508},
  year={2008},
  publisher={Elsevier}
}

@inproceedings{
heurtebise2026multiview,
title={Multi-View Causal Discovery without Non-Gaussianity: Identifiability and Algorithms},
author={Ambroise Heurtebise and Omar Chehab and Pierre Ablin and Alexandre Gramfort and Aapo Hyvarinen},
booktitle={Forty-third International Conference on Machine Learning},
year={2026}
}

@inproceedings{kingma2014adam,
  title     = {Adam: A Method for Stochastic Optimization},
  author    = {Kingma, Diederik P and Ba, Jimmy},
  booktitle = {International Conference on Learning Representations},
  year      = {2015}
}

@inproceedings{gobler2024texttt,
  title={causalAssembly: Generating Realistic Production Data for Benchmarking Causal Discovery},
  author={G{\"o}bler, Konstantin and Windisch, Tobias and Drton, Mathias and Pychynski, Tim and Roth, Martin and Sonntag, Steffen},
  booktitle={Causal Learning and Reasoning},
  pages={609--642},
  year={2024},
  organization={PMLR}
}

@article{scholkopf2021toward,
  title={Toward causal representation learning},
  author={Sch{\"o}lkopf, Bernhard and Locatello, Francesco and Bauer, Stefan and Ke, Nan Rosemary and Kalchbrenner, Nal and Goyal, Anirudh and Bengio, Yoshua},
  journal={Proceedings of the IEEE},
  volume={109},
  number={5},
  pages={612--634},
  year={2021},
  publisher={IEEE}
}

@book{spirtes2001causation,
  title={Causation, Prediction, and Search},
  author={Spirtes, Peter and Glymour, Clark N. and Scheines, Richard},
  year={2001},
  publisher={MIT Press}
}

@article{zheng2018dags,
  title={Dags with no tears: Continuous optimization for structure learning},
  author={Zheng, Xun and Aragam, Bryon and Ravikumar, Pradeep K and Xing, Eric P},
  journal={Advances in neural information processing systems},
  volume={31},
  year={2018}
}

@article{pearl2003causality,
  title={Causality: models, reasoning, and inference},
  author={Pearl, Judea and others},
  journal={Econometric Theory},
  volume={19},
  number={675-685},
  pages={46},
  year={2003},
  publisher={Cambridge university press}
}

@book{peters2017elements,
  title={Elements of causal inference: foundations and learning algorithms},
  author={Peters, Jonas and Janzing, Dominik and Sch{\"o}lkopf, Bernhard},
  year={2017},
  publisher={The MIT press}
}

@article{gallego2025cooper,
  title={Cooper: A library for constrained optimization in deep learning},
  author={Gallego-Posada, Jose and Ramirez, Juan and Hashemizadeh, Meraj and Lacoste-Julien, Simon},
  journal={arXiv preprint arXiv:2504.01212},
  year={2025}
}
\bibliographystyle{abbrvnat}

\clearpage
\appendix
\thispagestyle{empty}

\onecolumn
\aistatstitle{Causal Discovery for Linear Models with Latent Confounding via Multi-domain Identifiable Representation Learning: \\Supplementary Materials}

\section{Summary of Notation}
\label{app: notation}
Throughout this paper, we adopt the following conventions for mathematical notation: We use bold uppercase letters such as $\Xb$ to denote random vectors, and uppercase italic letters such as $X$ to denote matrices. We list all specific
notations in the following table:
\begin{table}[ht]
\label{tab: notation}
\begin{center}
\begin{tabular}{cc}
\hline
\textbf{Symbol}  & \textbf{Description} \\ \hline
$\tau$   & Domain index \\
$\Ub$   & Unobserved variables\\
$\Xb$   & Observed variables \\
$\Eb_1$   & Exogenous variables of $\Xb$\\
$\Eb_2$   & Exogenous variables of $\Ub$\\
$\Zb^\tau$   & Latent noise variables from domain $\tau$\\
$\Xb^\tau$   & Observed variables from domain $\tau$ \\
$A$   & Linear mixing matrix \\
$B$   & Weighted adjacency matrix of directed causal graph among observed variables\\
$\Gcal$   & ADMG\\
$E_{\rightarrow}$   & directed edge set\\
$E_{\leftrightarrow}$   & bidirected edge set\\
$n$ & Dimension of the observed variables \\
$m$ & Dimension of the unobserved variables \\
$T$   & Number of domains\\
$N$   & Number of samples per domain\\
\hline
\end{tabular}

\end{center}
\end{table}

\section{Details of LiNGAM-Based Directed Structure Recovery}
\label{app:lingam-postprocessing}
In this section, we provide the details of the LiNGAM-style post-processing procedure used to recover the directed structure from the estimated mixing matrix $\hat{A}$. We break it into 3 steps to resolve the permutation, scaling, and numerical error.

We start by introducing new notation. Let $W:=A^{-1}=I-B$. Inverting $\widehat A$ gives
\begin{equation}
\widetilde W
:=
\widehat A^{-1}
=
D^{-1}P^T W.
\label{eq:estimated-demixing}
\end{equation}
Hence, $\widetilde W$ recovers the structural demixing matrix $W$ up
to a row permutation and row-wise scaling.

\paragraph{Step i: Resolve the permutation ambiguity.}
By assuming $B$ is a DAG weighted adjacency matrix, we have $B_{ii}=0$, which implies $W=I-B$ must have exactly one diagonal entries, i.e. $W_{ii}=0$. We therefore search a row permutation of $\widetilde W$ that aligns the recovered components with the observed variables.

Specifically, we choose the permutation $\widehat P$ solving
\begin{equation}
\widehat P
\in
\argmax_{P\in\Pcal_n}
\sum_{i=1}^n
\left|
(P\widetilde W)_{ii}
\right|,
\label{eq:assignment}
\end{equation}
where $\Pcal_n$ denotes the set of $n\times n$ permutation matrices. This optimization is a linear assignment problem and can be solved efficiently using, for example, the Hungarian algorithm with complexity $O(n^3)$. We then define
\begin{equation}
\overline W
:=
\widehat P\widetilde W.
\end{equation}

\paragraph{Step ii: Remove the scaling ambiguity.}
After permutation alignment, the remaining ambiguity is row-wise scaling. Since $W=I-B$ satisfies $W_{ii}=1$, we normalize each row by its diagonal entry:
\begin{equation}
\widehat W_{ij}
:=
\frac{\overline W_{ij}}
{\overline W_{ii}},
\qquad i,j\in[n].
\label{eq:row-normalization}
\end{equation}

\paragraph{Step iii: Recover the directed structural coefficients.}
We obtain the estimated weighted adjacency matrix as
\begin{equation}
\widetilde B
:=
I-\widehat W.
\label{eq:B-estimate}
\end{equation}

\textbf{Sparsification.} However, finite-sample estimation generally produces small spurious nonzero coefficients. We therefore sparsify $\widetilde B$ using either an absolute threshold,
\begin{equation}
\widetilde B_{ij}
\leftarrow 0
\quad\text{if}\quad
|\widetilde B_{ij}|<\lambda,
\end{equation}
or a relative threshold defined with respect to the magnitude of the recovered coefficients.

\textbf{Acyclicity.} Thresholding alone need not yield an acyclic directed graph. Following \cite{shimizu2006linear}, if directed cycles remain, we iteratively remove the weakest edge participating in a directed cycle until it becomes acyclic. Denote the resulting matrix by $\widehat B$. The estimated directed edge set is then
\begin{equation}
\widehat E_{\to}
:=
\{
X_j\to X_i:
\widehat B_{ij}\neq0
\}.
\end{equation}

\section{Proof}
\label{app: proof}

\subsection{Proof of \cref{lem:directed_identification}}
\label{app: proof_lemma_1}
Before proving Lemma~\ref{lem:directed_identification}, we recall the following result underlying the LiNGAM postprocessing procedure.

\begin{lemma}[LiNGAM permutation lemma {\citealp[Lemma~1]{shimizu2006linear}}]
\label{lem: LiNGAM}
Let $M$ be a lower-triangular matrix whose diagonal entries are all nonzero.
Among matrices obtained from $M$ by independently permuting its rows and columns, the diagonal contains no zero entries if and only if the row and column permutations coincide.

Consequently, if a matrix $W$ can be made lower triangular with nonzero diagonal by applying the same permutation to its rows and columns, then, for any row-permuted version of $W$, there exists a unique row permutation that restores a matrix with an everywhere nonzero diagonal.
\end{lemma}

The second statement is the form of Lemma~\ref{lem: LiNGAM} used in the LiNGAM postprocessing procedure. After the correct row permutation is identified, row-wise scaling is removed by normalizing each row by its diagonal entry.
\directedIdentification*
\begin{proof}
From $\widehat A=APD$, and $A^{-1}=I-B$, we obtain $\widehat A^{-1}=D^{-1}P^T(I-B)$.

Thus, $\widehat A^{-1}$ differs from $I-B$ only by a row permutation and nonzero row-wise scaling.

Since $B$ represents a DAG, there exists a permutation of the variables under which $B$ is strictly lower triangular. Hence, under the same permutation, $I-B$ is lower triangular with all diagonal entries equal to one. By Lemma~\ref{lem: LiNGAM}, there is a unique row permutation of $\widehat A^{-1}$ whose diagonal entries are all nonzero. Therefore, the LiNGAM permutation step recovers the correct correspondence between the rows of $\widehat A^{-1}$ and the observed variables.

After correcting the row permutation, the resulting matrix differs from $I-B$ only by row-wise scaling. Since the diagonal of $I-B$ is equal to one, normalizing each row by its diagonal entry removes this scaling ambiguity exactly. Therefore, the normalized demixing matrix satisfies $\widehat W=I-B$, and consequently
\begin{equation}
\widehat B
=
I-\widehat W
=
B.
\end{equation}

It follows immediately that
\begin{equation}
\widehat E_{\rightarrow}
=
E_{\rightarrow}.
\end{equation}

Next, we show that $\widehat E^\tau_{\leftrightarrow}
\subseteq E^\tau_{\leftrightarrow}$. We prove this by contradiction. Suppose 
\begin{equation}
\label{equ: contra}
\exists\quad X_i^{\tau}\leftrightarrow X_j^{\tau} \in \widehat E^\tau_{\leftrightarrow}, \text{ such that }X_i^{\tau}\leftrightarrow X_j^{\tau}\not\in E^\tau_{\leftrightarrow}.
\end{equation}

Since the bidirected-edge recovery procedure adds an edge only when the corresponding estimated disturbance components are dependent, we must have
\begin{equation}
\widehat Z_i^\tau
\not\!\perp\!\!\!\perp
\widehat Z_j^\tau.
\end{equation}

From the previous result, we know $\widehat B=B$, and therefore
\begin{equation}
\widehat \Zb^\tau
=
(I-\widehat B)\Xb^\tau
=
(I-B)\Xb^\tau
=
\Zb^\tau.
\end{equation}
Thus we have,
\begin{equation}
Z_i^\tau
\not\!\perp\!\!\!\perp
Z_j^\tau.
\end{equation}

To obtain a contradiction, recall from Eq.~\ref{eq:base-sem} that
\begin{equation}
\Zb^\tau=B_1^\tau
(I-B_2^\tau)^{-1}
\Eb_2^\tau
+
\Eb_1^\tau.
\end{equation}
Define $M^\tau:=B_1^\tau(I-B_2^\tau)^{-1}$. We write it out component by component:
\begin{equation}
Z_i^\tau
=
\sum_k
M^\tau_{ik}\Eb_{2k}^\tau
+
\Eb_{1i}^\tau,
\qquad
Z_j^\tau
=
\sum_k
M^\tau_{jk}\Eb_{2k}^\tau
+
\Eb_{1j}^\tau.
\end{equation}
$M^\tau_{ik}$ represents the total linear effect of the latent exogenous component $\Eb_{2k}^\tau$ on the observed disturbance $Z_i^\tau$. Thus, a common latent source can induce dependence between $Z_i^\tau$ and $Z_j^\tau$ only if there exists some $k$ such that
\begin{equation}
M^\tau_{ik}\neq0
\qquad\text{and}\qquad
M^\tau_{jk}\neq0.
\end{equation}

By construction of the domain-specific ADMG, such a shared latent source induces the bidirected edge $X_i\leftrightarrow X_j \in E^\tau_{\leftrightarrow}$, which is contradictory with $
X_i\leftrightarrow X_j
\notin
E^\tau_{\leftrightarrow}
$
in Equ.~\ref{equ: contra}. Therefore, we conclude
\begin{equation}
\widehat E^\tau_{\leftrightarrow}
\subseteq
E^\tau_{\leftrightarrow}.
\end{equation}
\end{proof}

\subsection{Proof of Theorem~\ref{thm:linadmg-recovery}}
\label{app: proof_thm}
\linadMGrecovery*

\begin{proof}
By Lemma~\ref{lem:directed_identification}, the directed structural coefficient matrix is identified:
\begin{equation}
\widehat B=B.
\end{equation}
Therefore, $\widehat E_{\rightarrow}=E_{\rightarrow}$

Moreover, Lemma~\ref{lem:directed_identification} shows that the estimated bidirected edge set contains no false-positive edges:
\begin{equation}
\widehat E^\tau_{\leftrightarrow}
\subseteq
E^\tau_{\leftrightarrow}.
\label{eq:bidir-subset}
\end{equation}
Therefore, we only need to show the other way around: $ E^\tau_{\leftrightarrow}
\subseteq
\widehat E^\tau_{\leftrightarrow}$.

Since $\widehat B=B$, the composite disturbances are also recovered exactly:
\begin{equation}
\widehat{\Zb}^{\tau}
=
(I-\widehat B)\Xb^\tau
=
(I-B)\Xb^\tau
=
\Zb^\tau.
\label{eq:exact-Z-recovery}
\end{equation}

Now consider an arbitrary true bidirected edge $X_i^{\tau}\leftrightarrow X_j^{\tau}
\in
E^\tau_{\leftrightarrow}$.
By Assump.~\ref{assump: no-cancel}, this implies 
\begin{equation}
Z_i^\tau
\not\!\perp\!\!\!\perp
Z_j^\tau.
\label{eq:true-Z-dependence}
\end{equation}

By Equ.~\eqref{eq:exact-Z-recovery}, we then have
$
\widehat Z_i^\tau
=
Z_i^\tau$ and $\widehat Z_j^\tau=Z_j^\tau$,
and therefore $
\widehat Z_i^\tau
\not\!\perp\!\!\!\perp
\widehat Z_j^\tau$.

Thus, under the dependence criterion used by LinADMG, $X_i^{\tau}\leftrightarrow X_j^{\tau}
\in
\widehat E^\tau_{\leftrightarrow}$.

Since this is true for all $X_i^{\tau}\leftrightarrow X_j^{\tau}\in E^\tau_{\leftrightarrow}$, we obtain
\begin{equation}
E^\tau_{\leftrightarrow}
\subseteq
\widehat E^\tau_{\leftrightarrow}.
\label{eq:bidir-superset}
\end{equation}
Combining Equs.~\eqref{eq:bidir-subset} and~\eqref{eq:bidir-superset}, we can conclude
\begin{equation}
\widehat E^\tau_{\leftrightarrow}
=
E^\tau_{\leftrightarrow}.
\end{equation}

\end{proof}
\FloatBarrier

\section{Identifiable and causal representation learning.}
\label{app: IRL and CRL}

A broad literature studies conditions under which latent representations become identifiable from observed data. In blind source separation, source independence can sometimes be relaxed by assuming additional temporal structure \citep{hyvarinen2004blind,gomez2008measuring,zhang2011general}. Nonlinear ICA
further uses auxiliary variables, temporal structure, or nonstationarity to identify otherwise non-identifiable latent components \citep{hyvarinen2017nonlinear,hyvarinen2019nonlinear,
khemakhem2020variational,kivva2022identifiability}. More recent causal representation learning (CRL) methods leverage additional information such as interventions \citep{brehmer2022weakly,lippe2022citris,von2023nonparametric,
buchholz2023learning,zhang2023identifiability,ahuja2023interventional}, temporal structure \citep{lippe2023biscuit,lachapelle2022disentanglement,
lachapelle2024nonparametric,li2025towards}, multi-view observations
\citep{yao2023multi,ahuja2024multi,morioka2023causal}, hierarchical structure \citep{kong2023identification,kong2024learning}, or restrictions on the latent causal graph
\citep{zhang2024causal,dong2024on,song2024causal}. Particularly relevant to our setting are linear representation-learning methods based on hard interventions \citep{varici2024linear}, multi-domain scaling
\citep{xu2026identifying}, and sparse latent support \citep{xu2024sparsity}.

These works primarily aim to identify latent representations and, in some settings, causal relations among the latent variables. LinADMG only uses representation identifiability as an intermediate step toward recovering causal structure among the observed variables. Among them, Multi-view cannot be converted to ADMG learning under the LinADMG framework. Different views generally have distinct mixing functions, and the recovered representations may be identifiable only up to view-specific permutation and component-wise scaling ambiguities. As a result, there need not exist a single shared mixing matrix corresponding to an invariant directed structure, as required by our framework.

\FloatBarrier

\section{Implementation Details and Licenses}
\label{app: imple and license}
This section provides further details about the experiment implementation in \cref{sec: experiment result}. Our synthetic-data generation utilities build on the open-source NOTEARS implementation by \citet{zheng2018dags}, distributed under the Apache License 2.0. We use these utilities to sample random causal graphs, structural coefficients, and the latent variables used by our synthetic data generators. Our implementation of UMNI adapts from the repository released by \citet{varici2024linear}, which is distributed under
the Apache License 2.0. Our BANG implementation is a Python port of the R implementation from \citet{wang2023causal}, which is distributed under the
MIT License. The DirectLiNGAM and RCD baselines are evaluated using the implementations provided by the \texttt{lingam} Python package
\citep{shimizu2006linear,maeda2020rcd}, which is distributed under the
MIT License. The FCI baseline and the conversion of fitted ADMGs to PAGs use the \texttt{causal-learn} package, distributed under the MIT License. The
implementation released by \citet{xu2024sparsity} uses the constrained-optimization library \texttt{Cooper} \citep{gallego2025cooper}, also distributed under the MIT License. For CausalAssembly, we use version 1.2.1 of the \texttt{causalAssembly} package, distributed under the GNU Affero General Public License v3.0.

We list all the hyperparameters we used for training in Tab.~\ref{tab:method-hyperparameters}, and provide all the code for our method.
-MuDonll, -diagGMM, and VaDE-linear use Adam \citep{kingma2014adam}; -SP uses Cooper’s ExtraAdam for primal and dual updates; DCD uses L-BFGS-B within an augmented-Lagrangian procedure. Experiments were run on CPU nodes with dual AMD EPYC 7H12 processors (128 cores per node) and GPU nodes with NVIDIA A100 GPUs (40GB each).

\begin{table}[t]
\caption{
Hyperparameters used in the experiments for all methods and baselines.
}
\label{tab:method-hyperparameters}
\centering
\footnotesize
\begin{tabularx}{\textwidth}{
    @{}
    l
    >{\raggedright\arraybackslash}X
    @{}
}
\toprule
Method & Hyperparameters \\
\midrule

\texttt{LinADMG-SP} 
&

\texttt{lr}=$10^{-4}$;
\texttt{num\_steps}=15000;
\texttt{num\_initializations}=1;
\texttt{batch\_size}=5000;
\texttt{sparse\_level}=0.01;
\texttt{aug\_lag\_coef}=0
\\[2pt]

\texttt{LinADMG-MuDoNLL} 
&

\texttt{lr}=0.01;
\texttt{num\_steps}=5000;
\texttt{num\_initializations}=1;
\texttt{scheduler}=\texttt{cosine};
\texttt{grad\_clip}=1;
\\[2pt]

\texttt{LinADMG-UMNI} 
&

\texttt{atol\_eigv}=0.05;
\texttt{atol\_ci\_test}=0.05;
\texttt{kappa}=1;
\texttt{umni\_batch\_size}=2048;
\\[2pt]

\texttt{LinADMG-diagGMM} 
&

\texttt{lr}=0.001;
\texttt{num\_steps}=5000;
\texttt{num\_initializations}=1;
\texttt{batch\_size}=6144;
\\[2pt]

\midrule

LiNGAM
&
\texttt{lingam\_measure}=\texttt{pwling};
\\[2pt]

BANG
&
\texttt{bang\_moment\_degree}=3;
\texttt{bang\_level}=0.01;
\texttt{bang\_restriction}=1;
\texttt{bang\_max\_iterations}=10000;
\texttt{bang\_condition\_limit}=$10^{12}$;
\\[2pt]

DCD
&
\texttt{dcd\_lambda}=0.05;
\texttt{dcd\_threshold}=0.05;
\texttt{dcd\_num\_restarts}=5;
\\[2pt]

RCD
&
\texttt{rcd\_max\_samples}=300;
\texttt{rcd\_cor\_alpha}=0.01;
\texttt{rcd\_ind\_alpha}=0.01;
\texttt{rcd\_shapiro\_alpha}=0.01;
\texttt{rcd\_independence}=\texttt{hsic};
\\

\bottomrule
\end{tabularx}
\end{table}

\FloatBarrier

\section{Additional Experiments results}
\label{app: experiments}

\subsection{Esitmation of Directed-edge weights $B$}
\label{app: B est}
\begin{figure}
    \centering
    \includegraphics[width=0.3\linewidth]{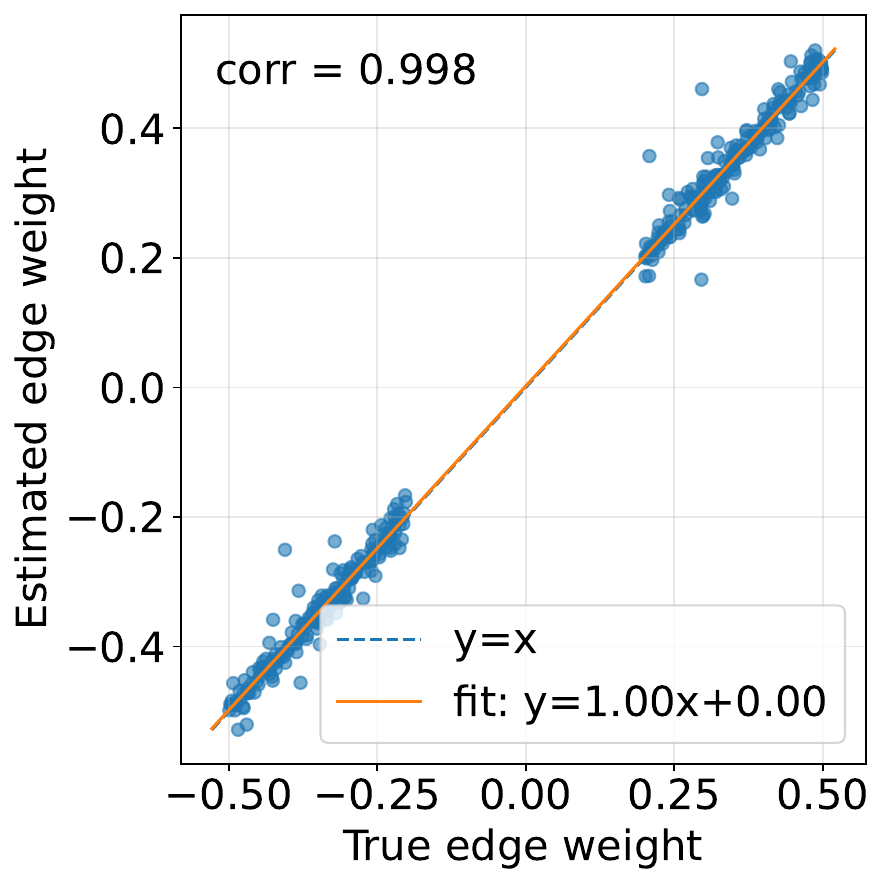}
    \caption{True vs. estimated directed-edge weights by \texttt{LinADMG-MuDoNLL} under multi-domain scalings DGP. 
    }
    \label{fig:nll_true_vs_estimated}
\end{figure}

We conduct the experiment with \texttt{LinADMG-MuDoNLL}  on Multi-domain Scalings DGP. In every experiment, we set $n=10, N=5000, T=1, k_X=2.0, k_U=0.5$ and Gaussian noise. Results are averaged over 20 random seeds. In Fig.~\ref{fig:nll_true_vs_estimated}, the estimated directed-edge weights closely follow their ground-truth values, concentrating around the identity line. The near-perfect correlation ($0.998$) and fitted relation $y=x$ show highly accurate estimation.

\subsection{All configurations of synthetic data experiments}
\label{app: all config}
We evaluate the estimation of both the directed and bidirected edges of the underlying ADMG under four multi-domain DGPs. The UP-Observation DGP applies domain-specific binary masks to the latent variables, the MD-Scalings DGP applies domain-specific continuous scaling factors, the hard-intervention DGP replaces the structural equations of intervened latent variables, and the Diag-GMMs DGP generates five groups whose latent distributions vary with an unobserved auxiliary variable. For each DGP, its designated method: \texttt{LinADMG-SP} , -MuDoNll, -UMNI, and -diagGMM, respectively, is compared with all four baselines LiNGAM, BANG, DCD, and RCD. 

The anchor configuration uses $n=10$, $T=n$ domains, ER graphs with directed-edge density $k_X=2$ and confounding density $k_U=0.5$, and Gaussian noise. Each domain contains $N=5000$ observations, except for the hard-intervention anchor, which uses $N=50000$. We conduct one-factor-at-a-time studies over dimension $n\in\{5,10,20\}$, number of domains $T/n\in\{1,2,3\}$, sample size $N\in\{1000,5000,50000\}$, directed-edge density $k_X\in\{1,2,3\}$, confounding density $k_U\in\{0,0.5,1\}$, and Gaussian, exponential, and Gumbel noise. The Diag-GMMs DGP always contains five Gaussian groups, so the domain, confounding-density, and noise sweeps do not apply to it. We report the mean and sample standard deviation of edge-wise F1, for which larger values are better, and structural Hamming distance (SHD), for which smaller values are better, over 20 independently generated random seeds.

\paragraph{MD-Scalings}
Tables~\ref{tab:scaling-f1} and~\ref{tab:scaling-shd} show that \texttt{LinADMG-MuDoNLL} benefits substantially from domain-specific scaling and is the strongest method in most configurations. Its performance improves systematically with the amount of multi-domain information: increasing the number of domains from $T=n$ to $T=3n$ raises directed F1 and bidirected F1, while increasing $N$ from $1000$ to $50000$ raises the respective scores as well. LiNGAM is strongest when there is no confounding. In contrast, BANG, DCD, and RCD recover little of the graph under Gaussian MD-Scalings, although BANG and RCD improve under non-Gaussian noise.

\paragraph{Diag-GMMs.}
Tables~\ref{tab:auxillary-f1} and~\ref{tab:auxillary-shd} show \texttt{LinADMG-diagGMM}  consistently provides the best directed estimation across all configurations. For bidirected edges, \texttt{LinADMG-diagGMM}  and BANG are the two strongest methods. BANG is slightly better at $n=5$ and $n=10$ and approaches perfect recovery at $N=50000$. Overall, \texttt{LinADMG-diagGMM}  offers the best overall recovery of the complete ADMG, while BANG is especially effective for the complete bidirected structure
when sufficient data are available. Results of BANG and DCD at $n=20$ are missing due to computational cost.

\paragraph{UP-Observations.} Tables~\ref{tab:mask-f1} and~\ref{tab:mask-shd} show that LiNGAM provides the strongest overall performance under the UP-Observations DGP for directed edge estimation. \texttt{LinADMG-SP}  is less accurate for bidirected-edge estimation and shows relatively large SHD, indicating that it tends to introduce additional bidirected edges. Nevertheless, \texttt{LinADMG-SP}  becomes competitive for directed recovery under non-Gaussian noise and achieves the best directed results under Gumbel noise, whereas BANG is best under exponential noise. Performance generally decreases as the dimension $n$ and confounding density $K_U$ increase, although the increase in bidirected F1 at higher confounding density partly reflects the larger number of true bidirected edges. Overall, LiNGAM in general performs the best across the configurations, since the masking layer introduces independence to the noises, while the main advantage of \texttt{LinADMG-SP} is limited to directed recovery under Gumbel noise.

\paragraph{UMN-Hard Interventions.} Tables~\ref{tab:hard-intervention-f1} and \ref{tab:hard-intervention-shd} shows when $n$ is relatively small, UMNI bidirected edges. At the $n=10$ Gaussian anchor, however, its accuracy decreases and becomes
highly variable. LiNGAM gives the best directed recovery, whereas UMNI better detects true confounders and avoids false positives. Under exponential and Gumbel noise, BANG becomes the strongest method for both directed and bidirected estimation, consistent with its assumption of non-Gaussianity. UMNI results are missing for the setting $n=20$, $T=3n$, $N=1000$, and $N=5000$ due to the requirement of sufficient sample numbers for precise score estimation from the algorithm perspective.
\begin{table*}[t]
\caption{MD-Scalings results: average F1 $\pm$ standard deviation over 20 random seeds. The varying parameter and the best method in each row are shown in bold.}
\label{tab:scaling-f1}
\centering
\begin{small}
\begin{tabular}{llccccc}
\toprule
Setting & Value & LinADMG-MuDoNLL & LiNGAM & BANG & DCD & RCD \\
\midrule
\multicolumn{7}{l}{\textit{Directed edges}} \\
\multicolumn{7}{l}{\emph{Dimension}} \\
$n$ & \textbf{5} & 0.927 $\pm$ 0.125 & \textbf{0.943 $\pm$ 0.107} & 0.000 $\pm$ 0.000 & 0.156 $\pm$ 0.139 & 0.197 $\pm$ 0.173 \\
$n$ & \textbf{10} & \textbf{0.883 $\pm$ 0.196} & 0.815 $\pm$ 0.096 & 0.005 $\pm$ 0.021 & 0.078 $\pm$ 0.073 & 0.064 $\pm$ 0.083 \\
$n$ & \textbf{20} & \textbf{0.854 $\pm$ 0.280} & 0.754 $\pm$ 0.069 & 0.005 $\pm$ 0.015 & 0.026 $\pm$ 0.034 & 0.033 $\pm$ 0.038 \\
\addlinespace
\multicolumn{7}{l}{\emph{Number of domains}} \\
$T/n$ & \textbf{1} & \textbf{0.883 $\pm$ 0.196} & 0.815 $\pm$ 0.096 & 0.005 $\pm$ 0.021 & 0.078 $\pm$ 0.073 & 0.064 $\pm$ 0.083 \\
$T/n$ & \textbf{2} & \textbf{0.911 $\pm$ 0.171} & 0.825 $\pm$ 0.063 & 0.000 $\pm$ 0.000 & 0.053 $\pm$ 0.066 & 0.041 $\pm$ 0.073 \\
$T/n$ & \textbf{3} & \textbf{0.976 $\pm$ 0.087} & 0.831 $\pm$ 0.041 & 0.005 $\pm$ 0.021 & 0.082 $\pm$ 0.078 & 0.052 $\pm$ 0.102 \\
\addlinespace
\multicolumn{7}{l}{\emph{Sample size}} \\
$N$ & \textbf{1000} & \textbf{0.843 $\pm$ 0.101} & 0.821 $\pm$ 0.072 & 0.010 $\pm$ 0.029 & 0.068 $\pm$ 0.062 & 0.079 $\pm$ 0.105 \\
$N$ & \textbf{5000} & \textbf{0.883 $\pm$ 0.196} & 0.815 $\pm$ 0.096 & 0.005 $\pm$ 0.021 & 0.078 $\pm$ 0.073 & 0.064 $\pm$ 0.083 \\
$N$ & \textbf{50000} & \textbf{0.967 $\pm$ 0.108} & 0.818 $\pm$ 0.081 & 0.000 $\pm$ 0.000 & 0.063 $\pm$ 0.051 & 0.063 $\pm$ 0.109 \\
\addlinespace
\multicolumn{7}{l}{\emph{Graph density}} \\
$k_X$ & \textbf{1.0} & \textbf{0.773 $\pm$ 0.111} & 0.671 $\pm$ 0.097 & 0.000 $\pm$ 0.000 & 0.083 $\pm$ 0.086 & 0.122 $\pm$ 0.125 \\
$k_X$ & \textbf{2.0} & \textbf{0.883 $\pm$ 0.196} & 0.815 $\pm$ 0.096 & 0.005 $\pm$ 0.021 & 0.078 $\pm$ 0.073 & 0.064 $\pm$ 0.083 \\
$k_X$ & \textbf{3.0} & 0.818 $\pm$ 0.268 & \textbf{0.876 $\pm$ 0.099} & 0.003 $\pm$ 0.014 & 0.066 $\pm$ 0.074 & 0.052 $\pm$ 0.082 \\
\addlinespace
\multicolumn{7}{l}{\emph{Confounding density}} \\
$k_U$ & \textbf{0.0} & 0.909 $\pm$ 0.063 & \textbf{0.999 $\pm$ 0.005} & 0.000 $\pm$ 0.000 & 0.089 $\pm$ 0.063 & 0.079 $\pm$ 0.104 \\
$k_U$ & \textbf{0.5} & \textbf{0.883 $\pm$ 0.196} & 0.815 $\pm$ 0.096 & 0.005 $\pm$ 0.021 & 0.078 $\pm$ 0.073 & 0.064 $\pm$ 0.083 \\
$k_U$ & \textbf{1.0} & \textbf{0.803 $\pm$ 0.277} & 0.674 $\pm$ 0.097 & 0.005 $\pm$ 0.020 & 0.053 $\pm$ 0.076 & 0.099 $\pm$ 0.080 \\
\addlinespace
\multicolumn{7}{l}{\emph{Noise distribution}} \\
Noise & \textbf{gauss} & \textbf{0.883 $\pm$ 0.196} & 0.815 $\pm$ 0.096 & 0.005 $\pm$ 0.021 & 0.078 $\pm$ 0.073 & 0.064 $\pm$ 0.083 \\
Noise & \textbf{exp} & \textbf{0.850 $\pm$ 0.176} & 0.788 $\pm$ 0.107 & 0.823 $\pm$ 0.070 & 0.098 $\pm$ 0.084 & 0.696 $\pm$ 0.096 \\
Noise & \textbf{gumbel} & \textbf{0.832 $\pm$ 0.218} & 0.763 $\pm$ 0.091 & 0.424 $\pm$ 0.144 & 0.066 $\pm$ 0.083 & 0.618 $\pm$ 0.142 \\
\midrule
\multicolumn{7}{l}{\textit{Bidirected edges}} \\
\multicolumn{7}{l}{\emph{Dimension}} \\
$n$ & \textbf{5} & \textbf{0.576 $\pm$ 0.142} & --  & 0.000 $\pm$ 0.000 & 0.190 $\pm$ 0.304 & 0.364 $\pm$ 0.128 \\
$n$ & \textbf{10} & \textbf{0.783 $\pm$ 0.142} & --  & 0.000 $\pm$ 0.000 & 0.056 $\pm$ 0.129 & 0.306 $\pm$ 0.080 \\
$n$ & \textbf{20} & \textbf{0.857 $\pm$ 0.156} & --  & 0.006 $\pm$ 0.028 & 0.018 $\pm$ 0.044 & 0.230 $\pm$ 0.048 \\
\addlinespace
\multicolumn{7}{l}{\emph{Number of domains}} \\
$T/n$ & \textbf{1} & \textbf{0.783 $\pm$ 0.142} & --  & 0.000 $\pm$ 0.000 & 0.056 $\pm$ 0.129 & 0.306 $\pm$ 0.080 \\
$T/n$ & \textbf{2} & \textbf{0.864 $\pm$ 0.225} & --  & 0.000 $\pm$ 0.000 & 0.011 $\pm$ 0.050 & 0.290 $\pm$ 0.102 \\
$T/n$ & \textbf{3} & \textbf{0.933 $\pm$ 0.156} & --  & 0.000 $\pm$ 0.000 & 0.012 $\pm$ 0.056 & 0.315 $\pm$ 0.069 \\
\addlinespace
\multicolumn{7}{l}{\emph{Sample size}} \\
$N$ & \textbf{1000} & \textbf{0.465 $\pm$ 0.092} & --  & 0.009 $\pm$ 0.041 & 0.036 $\pm$ 0.120 & 0.289 $\pm$ 0.088 \\
$N$ & \textbf{5000} & \textbf{0.783 $\pm$ 0.142} & --  & 0.000 $\pm$ 0.000 & 0.056 $\pm$ 0.129 & 0.306 $\pm$ 0.080 \\
$N$ & \textbf{50000} & \textbf{0.910 $\pm$ 0.185} & --  & 0.000 $\pm$ 0.000 & 0.035 $\pm$ 0.118 & 0.301 $\pm$ 0.082 \\
\addlinespace
\multicolumn{7}{l}{\emph{Graph density}} \\
$k_X$ & \textbf{1.0} & \textbf{0.777 $\pm$ 0.165} & --  & 0.000 $\pm$ 0.000 & 0.011 $\pm$ 0.050 & 0.402 $\pm$ 0.110 \\
$k_X$ & \textbf{2.0} & \textbf{0.783 $\pm$ 0.142} & --  & 0.000 $\pm$ 0.000 & 0.056 $\pm$ 0.129 & 0.306 $\pm$ 0.080 \\
$k_X$ & \textbf{3.0} & \textbf{0.658 $\pm$ 0.272} & --  & 0.000 $\pm$ 0.000 & 0.053 $\pm$ 0.112 & 0.265 $\pm$ 0.064 \\
\addlinespace
\multicolumn{7}{l}{\emph{Confounding density}} \\
$k_U$ & \textbf{0.0} & \textbf{0.000 $\pm$ 0.000} & --  & \textbf{0.000 $\pm$ 0.000} & \textbf{0.000 $\pm$ 0.000} & \textbf{0.000 $\pm$ 0.000} \\
$k_U$ & \textbf{0.5} & \textbf{0.783 $\pm$ 0.142} & --  & 0.000 $\pm$ 0.000 & 0.056 $\pm$ 0.129 & 0.306 $\pm$ 0.080 \\
$k_U$ & \textbf{1.0} & \textbf{0.823 $\pm$ 0.125} & --  & 0.000 $\pm$ 0.000 & 0.067 $\pm$ 0.083 & 0.555 $\pm$ 0.086 \\
\addlinespace
\multicolumn{7}{l}{\emph{Noise distribution}} \\
Noise & \textbf{gauss} & \textbf{0.783 $\pm$ 0.142} & --  & 0.000 $\pm$ 0.000 & 0.056 $\pm$ 0.129 & 0.306 $\pm$ 0.080 \\
Noise & \textbf{exp} & \textbf{0.764 $\pm$ 0.166} & --  & 0.581 $\pm$ 0.201 & 0.025 $\pm$ 0.112 & 0.241 $\pm$ 0.118 \\
Noise & \textbf{gumbel} & \textbf{0.714 $\pm$ 0.206} & --  & 0.376 $\pm$ 0.100 & 0.137 $\pm$ 0.162 & 0.243 $\pm$ 0.130 \\
\bottomrule
\end{tabular}
\end{small}
\end{table*}

\begin{table*}[t]
\caption{MD-Scalings results: average SHD $\pm$ standard deviation over 20 random seeds. The varying parameter and the best method in each row are shown in bold.}
\label{tab:scaling-shd}
\centering
\begin{small}
\begin{tabular}{llccccc}
\toprule
Setting & Value & LinADMG-MuDoNLL & LiNGAM & BANG & DCD & RCD \\
\midrule
\multicolumn{7}{l}{\textit{Directed edges}} \\
\multicolumn{7}{l}{\emph{Dimension}} \\
$n$ & \textbf{5} & 1.300 $\pm$ 2.227 & \textbf{1.100 $\pm$ 2.075} & 10.050 $\pm$ 0.224 & 9.900 $\pm$ 1.518 & 8.850 $\pm$ 1.089 \\
$n$ & \textbf{10} & \textbf{4.500 $\pm$ 5.708} & 9.300 $\pm$ 5.859 & 20.000 $\pm$ 0.324 & 21.050 $\pm$ 1.849 & 19.450 $\pm$ 0.999 \\
$n$ & \textbf{20} & \textbf{8.750 $\pm$ 12.586} & 26.550 $\pm$ 10.560 & 40.100 $\pm$ 0.447 & 42.850 $\pm$ 2.207 & 39.450 $\pm$ 0.759 \\
\addlinespace
\multicolumn{7}{l}{\emph{Number of domains}} \\
$T/n$ & \textbf{1} & \textbf{4.500 $\pm$ 5.708} & 9.300 $\pm$ 5.859 & 20.000 $\pm$ 0.324 & 21.050 $\pm$ 1.849 & 19.450 $\pm$ 0.999 \\
$T/n$ & \textbf{2} & \textbf{3.550 $\pm$ 6.419} & 8.400 $\pm$ 3.347 & 20.000 $\pm$ 0.000 & 20.550 $\pm$ 1.276 & 19.600 $\pm$ 0.754 \\
$T/n$ & \textbf{3} & \textbf{1.100 $\pm$ 4.012} & 8.050 $\pm$ 2.212 & 19.950 $\pm$ 0.224 & 20.200 $\pm$ 1.152 & 19.400 $\pm$ 1.231 \\
\addlinespace
\multicolumn{7}{l}{\emph{Sample size}} \\
$N$ & \textbf{1000} & \textbf{7.200 $\pm$ 4.360} & 8.700 $\pm$ 4.143 & 19.950 $\pm$ 0.394 & 20.250 $\pm$ 1.372 & 19.300 $\pm$ 1.174 \\
$N$ & \textbf{5000} & \textbf{4.500 $\pm$ 5.708} & 9.300 $\pm$ 5.859 & 20.000 $\pm$ 0.324 & 21.050 $\pm$ 1.849 & 19.450 $\pm$ 0.999 \\
$N$ & \textbf{50000} & \textbf{1.300 $\pm$ 4.207} & 9.100 $\pm$ 4.973 & 20.200 $\pm$ 0.410 & 20.750 $\pm$ 1.372 & 19.400 $\pm$ 1.142 \\
\addlinespace
\multicolumn{7}{l}{\emph{Graph density}} \\
$k_X$ & \textbf{1.0} & \textbf{6.050 $\pm$ 3.561} & 9.850 $\pm$ 3.937 & 10.000 $\pm$ 0.000 & 11.150 $\pm$ 1.631 & 9.400 $\pm$ 0.821 \\
$k_X$ & \textbf{2.0} & \textbf{4.500 $\pm$ 5.708} & 9.300 $\pm$ 5.859 & 20.000 $\pm$ 0.324 & 21.050 $\pm$ 1.849 & 19.450 $\pm$ 0.999 \\
$k_X$ & \textbf{3.0} & 9.000 $\pm$ 12.490 & \textbf{8.450 $\pm$ 7.529} & 29.950 $\pm$ 0.224 & 30.100 $\pm$ 1.683 & 29.150 $\pm$ 1.387 \\
\addlinespace
\multicolumn{7}{l}{\emph{Confounding density}} \\
$k_U$ & \textbf{0.0} & 4.050 $\pm$ 2.892 & \textbf{0.050 $\pm$ 0.224} & 20.000 $\pm$ 0.000 & 20.100 $\pm$ 1.165 & 19.250 $\pm$ 1.164 \\
$k_U$ & \textbf{0.5} & \textbf{4.500 $\pm$ 5.708} & 9.300 $\pm$ 5.859 & 20.000 $\pm$ 0.324 & 21.050 $\pm$ 1.849 & 19.450 $\pm$ 0.999 \\
$k_U$ & \textbf{1.0} & \textbf{7.350 $\pm$ 9.051} & 18.450 $\pm$ 6.395 & 20.100 $\pm$ 0.308 & 20.800 $\pm$ 1.609 & 19.200 $\pm$ 0.894 \\
\addlinespace
\multicolumn{7}{l}{\emph{Noise distribution}} \\
Noise & \textbf{gauss} & \textbf{4.500 $\pm$ 5.708} & 9.300 $\pm$ 5.859 & 20.000 $\pm$ 0.324 & 21.050 $\pm$ 1.849 & 19.450 $\pm$ 0.999 \\
Noise & \textbf{exp} & \textbf{6.050 $\pm$ 5.031} & 10.350 $\pm$ 5.706 & 7.300 $\pm$ 2.774 & 19.850 $\pm$ 1.461 & 10.350 $\pm$ 2.661 \\
Noise & \textbf{gumbel} & \textbf{6.650 $\pm$ 6.556} & 12.100 $\pm$ 5.389 & 17.500 $\pm$ 3.154 & 20.400 $\pm$ 1.465 & 12.050 $\pm$ 3.677 \\
\midrule
\multicolumn{7}{l}{\textit{Bidirected edges}} \\
\multicolumn{7}{l}{\emph{Dimension}} \\
$n$ & \textbf{5} & 3.650 $\pm$ 1.496 & --  & 2.400 $\pm$ 0.503 & \textbf{2.300 $\pm$ 0.865} & 5.700 $\pm$ 1.720 \\
$n$ & \textbf{10} & \textbf{4.500 $\pm$ 5.615} & --  & 7.500 $\pm$ 1.357 & 7.300 $\pm$ 1.490 & 21.500 $\pm$ 3.606 \\
$n$ & \textbf{20} & \textbf{5.550 $\pm$ 7.884} & --  & 16.000 $\pm$ 4.078 & 15.750 $\pm$ 3.932 & 64.750 $\pm$ 9.095 \\
\addlinespace
\multicolumn{7}{l}{\emph{Number of domains}} \\
$T/n$ & \textbf{1} & \textbf{4.500 $\pm$ 5.615} & --  & 7.500 $\pm$ 1.357 & 7.300 $\pm$ 1.490 & 21.500 $\pm$ 3.606 \\
$T/n$ & \textbf{2} & \textbf{3.000 $\pm$ 6.224} & --  & 7.050 $\pm$ 1.731 & 7.100 $\pm$ 1.832 & 21.050 $\pm$ 4.334 \\
$T/n$ & \textbf{3} & \textbf{1.600 $\pm$ 4.893} & -- & 7.150 $\pm$ 1.565 & 7.200 $\pm$ 1.673 & 22.200 $\pm$ 3.636 \\
\addlinespace
\multicolumn{7}{l}{\emph{Sample size}} \\
$N$ & \textbf{1000} & 15.250 $\pm$ 4.241 & --  & 7.500 $\pm$ 1.192 & \textbf{7.450 $\pm$ 1.432} & 22.150 $\pm$ 3.543 \\
$N$ & \textbf{5000} & \textbf{4.500 $\pm$ 5.615} & --  & 7.500 $\pm$ 1.357 & 7.300 $\pm$ 1.490 & 21.500 $\pm$ 3.606 \\
$N$ & \textbf{50000} & \textbf{2.300 $\pm$ 7.035} &--  & 7.500 $\pm$ 1.357 & 7.500 $\pm$ 1.433 & 22.650 $\pm$ 3.360 \\
\addlinespace
\multicolumn{7}{l}{\emph{Graph density}} \\
$k_X$ & \textbf{1.0} & \textbf{4.450 $\pm$ 5.558} & --  & 7.500 $\pm$ 1.277 & 7.450 $\pm$ 1.395 & 12.950 $\pm$ 2.328 \\
$k_X$ & \textbf{2.0} & \textbf{4.500 $\pm$ 5.615} & --  & 7.500 $\pm$ 1.357 & 7.300 $\pm$ 1.490 & 21.500 $\pm$ 3.606 \\
$k_X$ & \textbf{3.0} & 9.750 $\pm$ 10.543 & -- & \textbf{7.450 $\pm$ 1.317} & 7.750 $\pm$ 1.650 & 27.200 $\pm$ 3.679 \\
\addlinespace
\multicolumn{7}{l}{\emph{Confounding density}} \\
$k_U$ & \textbf{0.0} & 6.750 $\pm$ 3.985 & --  & \textbf{0.100 $\pm$ 0.308} & 0.250 $\pm$ 0.444 & 22.600 $\pm$ 5.185 \\
$k_U$ & \textbf{0.5} & \textbf{4.500 $\pm$ 5.615} & --  & 7.500 $\pm$ 1.357 & 7.300 $\pm$ 1.490 & 21.500 $\pm$ 3.606 \\
$k_U$ & \textbf{1.0} & \textbf{8.250 $\pm$ 6.766} & --  & 22.100 $\pm$ 4.128 & 21.450 $\pm$ 4.478 & 20.600 $\pm$ 3.119 \\
\addlinespace
\multicolumn{7}{l}{\emph{Noise distribution}} \\
Noise & \textbf{gauss} & \textbf{4.500 $\pm$ 5.615} & --  & 7.500 $\pm$ 1.357 & 7.300 $\pm$ 1.490 & 21.500 $\pm$ 3.606 \\
Noise & \textbf{exp} & 4.700 $\pm$ 4.438 & --  & \textbf{4.600 $\pm$ 1.875} & 7.500 $\pm$ 1.433 & 10.300 $\pm$ 2.342 \\
Noise & \textbf{gumbel} & \textbf{6.100 $\pm$ 6.095} &--  & 21.850 $\pm$ 5.896 & 8.750 $\pm$ 2.049 & 12.250 $\pm$ 3.401 \\
\bottomrule
\end{tabular}
\end{small}
\end{table*}

\begin{table*}[t]
\caption{Diag-GMMs results: average F1 $\pm$ standard deviation over 20 random seeds. The varying parameter and the best method in each row are shown in bold.}
\label{tab:auxillary-f1}
\centering
\begin{small}
\begin{tabular}{llccccc}
\toprule
Setting & Value & LinADMG-diagGMM & LiNGAM & BANG & DCD & RCD \\
\midrule
\multicolumn{7}{l}{\textit{Directed edges}} \\
\multicolumn{7}{l}{\emph{Dimension}} \\
$n$ & \textbf{5} & \textbf{0.873 $\pm$ 0.157} & 0.617 $\pm$ 0.193 & 0.000 $\pm$ 0.000 & 0.146 $\pm$ 0.132 & 0.148 $\pm$ 0.171 \\
$n$ & \textbf{10} & \textbf{0.817 $\pm$ 0.100} & 0.495 $\pm$ 0.128 & 0.062 $\pm$ 0.100 & 0.135 $\pm$ 0.083 & 0.141 $\pm$ 0.091 \\
$n$ & \textbf{20} & \textbf{0.743 $\pm$ 0.081} & 0.375 $\pm$ 0.058 & -- & -- & 0.134 $\pm$ 0.093 \\
\addlinespace
\multicolumn{7}{l}{\emph{Sample size}} \\
$N$ & \textbf{1000} & \textbf{0.806 $\pm$ 0.092} & 0.536 $\pm$ 0.126 & 0.109 $\pm$ 0.089 & 0.082 $\pm$ 0.079 & 0.154 $\pm$ 0.096 \\
$N$ & \textbf{5000} & \textbf{0.817 $\pm$ 0.100} & 0.495 $\pm$ 0.128 & 0.062 $\pm$ 0.100 & 0.135 $\pm$ 0.083 & 0.141 $\pm$ 0.091 \\
$N$ & \textbf{50000} & \textbf{0.821 $\pm$ 0.076} & 0.495 $\pm$ 0.108 & 0.005 $\pm$ 0.021 & 0.120 $\pm$ 0.083 & 0.160 $\pm$ 0.110 \\
\addlinespace
\multicolumn{7}{l}{\emph{Graph density}} \\
$k_X$ & \textbf{1.0} & \textbf{0.691 $\pm$ 0.107} & 0.313 $\pm$ 0.092 & 0.071 $\pm$ 0.108 & 0.087 $\pm$ 0.116 & 0.124 $\pm$ 0.112 \\
$k_X$ & \textbf{2.0} & \textbf{0.817 $\pm$ 0.100} & 0.495 $\pm$ 0.128 & 0.062 $\pm$ 0.100 & 0.135 $\pm$ 0.083 & 0.141 $\pm$ 0.091 \\
$k_X$ & \textbf{3.0} & \textbf{0.833 $\pm$ 0.109} & 0.653 $\pm$ 0.113 & 0.052 $\pm$ 0.061 & 0.125 $\pm$ 0.064 & 0.104 $\pm$ 0.084 \\
\midrule
\multicolumn{7}{l}{\textit{Bidirected edges}} \\
\multicolumn{7}{l}{\emph{Dimension}} \\
$n$ & \textbf{5} & 0.956 $\pm$ 0.077 & -- & \textbf{0.995 $\pm$ 0.016} & 0.538 $\pm$ 0.218 & 0.709 $\pm$ 0.211 \\
$n$ & \textbf{10} & 0.930 $\pm$ 0.198 &  -- & \textbf{0.946 $\pm$ 0.059} & 0.347 $\pm$ 0.101 & 0.673 $\pm$ 0.121 \\
$n$ & \textbf{20} & \textbf{0.951 $\pm$ 0.082} &  -- & -- & -- & 0.648 $\pm$ 0.104 \\
\addlinespace
\multicolumn{7}{l}{\emph{Sample size}} \\
$N$ & \textbf{1000} & \textbf{0.971 $\pm$ 0.049} & -- & 0.811 $\pm$ 0.134 & 0.393 $\pm$ 0.103 & 0.691 $\pm$ 0.121 \\
$N$ & \textbf{5000} & 0.930 $\pm$ 0.198 &  -- & \textbf{0.946 $\pm$ 0.059} & 0.347 $\pm$ 0.101 & 0.673 $\pm$ 0.121 \\
$N$ & \textbf{50000} & 0.973 $\pm$ 0.064 &  --& \textbf{0.996 $\pm$ 0.008} & 0.337 $\pm$ 0.065 & 0.637 $\pm$ 0.127 \\
\addlinespace
\multicolumn{7}{l}{\emph{Graph density}} \\
$k_X$ & \textbf{1.0} & 0.920 $\pm$ 0.151 &  -- & \textbf{0.940 $\pm$ 0.079} & 0.324 $\pm$ 0.094 & 0.658 $\pm$ 0.138 \\
$k_X$ & \textbf{2.0} & 0.930 $\pm$ 0.198 & -- & \textbf{0.946 $\pm$ 0.059} & 0.347 $\pm$ 0.101 & 0.673 $\pm$ 0.121 \\
$k_X$ & \textbf{3.0} & \textbf{0.921 $\pm$ 0.137} &  --& 0.918 $\pm$ 0.072 & 0.358 $\pm$ 0.087 & 0.751 $\pm$ 0.105 \\
\bottomrule
\end{tabular}
\end{small}
\end{table*}

\begin{table*}[t]
\caption{Diag-GMMs results: average SHD $\pm$ standard deviation over 20 random seeds. The varying parameter and the best method in each row are shown in bold.}
\label{tab:auxillary-shd}
\centering
\begin{small}
\begin{tabular}{llccccc}
\toprule
Setting & Value & LinADMG-diagGMM & LiNGAM & BANG & DCD & RCD \\
\midrule
\multicolumn{7}{l}{\textit{Directed edges}} \\
\multicolumn{7}{l}{\emph{Dimension}} \\
$n$ & \textbf{5} & \textbf{2.200 $\pm$ 2.546} & 7.550 $\pm$ 3.720 & 10.000 $\pm$ 0.000 & 9.950 $\pm$ 1.432 & 9.800 $\pm$ 1.508 \\
$n$ & \textbf{10} & \textbf{7.650 $\pm$ 3.897} & 30.300 $\pm$ 8.417 & 20.350 $\pm$ 1.599 & 21.200 $\pm$ 1.989 & 22.000 $\pm$ 2.026 \\
$n$ & \textbf{20} & \textbf{24.800 $\pm$ 9.339} & 114.450 $\pm$ 18.715 & -- & -- & 48.450 $\pm$ 4.548 \\
\addlinespace
\multicolumn{7}{l}{\emph{Sample size}} \\
$N$ & \textbf{1000} & \textbf{7.950 $\pm$ 3.634} & 26.450 $\pm$ 7.844 & 21.900 $\pm$ 2.049 & 21.900 $\pm$ 1.861 & 21.750 $\pm$ 2.807 \\
$N$ & \textbf{5000} & \textbf{7.650 $\pm$ 3.897} & 30.300 $\pm$ 8.417 & 20.350 $\pm$ 1.599 & 21.200 $\pm$ 1.989 & 22.000 $\pm$ 2.026 \\
$N$ & \textbf{50000} & \textbf{7.600 $\pm$ 3.331} & 31.600 $\pm$ 7.430 & 20.050 $\pm$ 0.394 & 21.650 $\pm$ 2.346 & 22.900 $\pm$ 2.827 \\
\addlinespace
\multicolumn{7}{l}{\emph{Graph density}} \\
$k_X$ & \textbf{1.0} & \textbf{8.550 $\pm$ 3.832} & 33.500 $\pm$ 6.747 & 10.800 $\pm$ 1.361 & 12.850 $\pm$ 2.346 & 13.800 $\pm$ 2.726 \\
$k_X$ & \textbf{2.0} & \textbf{7.650 $\pm$ 3.897} & 30.300 $\pm$ 8.417 & 20.350 $\pm$ 1.599 & 21.200 $\pm$ 1.989 & 22.000 $\pm$ 2.026 \\
$k_X$ & \textbf{3.0} & \textbf{9.550 $\pm$ 5.698} & 24.950 $\pm$ 8.532 & 31.050 $\pm$ 1.638 & 30.450 $\pm$ 2.305 & 31.200 $\pm$ 2.331 \\
\midrule
\multicolumn{7}{l}{\textit{Bidirected edges}} \\
\multicolumn{7}{l}{\emph{Dimension}} \\
$n$ & \textbf{5} & 0.750 $\pm$ 1.251 &  -- & \textbf{0.100 $\pm$ 0.308} & 6.050 $\pm$ 1.932 & 4.150 $\pm$ 2.323 \\
$n$ & \textbf{10} & \textbf{4.200 $\pm$ 9.876} &  -- & 4.350 $\pm$ 4.591 & 35.350 $\pm$ 3.360 & 21.600 $\pm$ 6.516 \\
$n$ & \textbf{20} & \textbf{16.050 $\pm$ 23.712} &  -- & -- & -- & 97.400 $\pm$ 21.960 \\
\addlinespace
\multicolumn{7}{l}{\emph{Sample size}} \\
$N$ & \textbf{1000} & \textbf{2.350 $\pm$ 3.951} &  -- & 13.450 $\pm$ 8.281 & 33.800 $\pm$ 3.427 & 20.700 $\pm$ 6.165 \\
$N$ & \textbf{5000} & \textbf{4.200 $\pm$ 9.876} &  -- & 4.350 $\pm$ 4.591 & 35.350 $\pm$ 3.360 & 21.600 $\pm$ 6.516 \\
$N$ & \textbf{50000} & 2.100 $\pm$ 4.767 &  -- & \textbf{0.350 $\pm$ 0.745} & 35.800 $\pm$ 2.093 & 23.450 $\pm$ 5.942 \\
\addlinespace
\multicolumn{7}{l}{\emph{Graph density}} \\
$k_X$ & \textbf{1.0} & 5.500 $\pm$ 8.593 &  -- & \textbf{4.650 $\pm$ 5.887} & 36.150 $\pm$ 3.066 & 22.300 $\pm$ 6.658 \\
$k_X$ & \textbf{2.0} & \textbf{4.200 $\pm$ 9.876} &  -- & 4.350 $\pm$ 4.591 & 35.350 $\pm$ 3.360 & 21.600 $\pm$ 6.516 \\
$k_X$ & \textbf{3.0} & \textbf{5.650 $\pm$ 8.061} &  -- & 6.500 $\pm$ 5.395 & 35.050 $\pm$ 2.929 & 17.500 $\pm$ 5.862 \\
\bottomrule
\end{tabular}
\end{small}
\end{table*}

\begin{table*}[t]
\caption{UP-Observations results: average F1 $\pm$ standard deviation over 20 random seeds. The varying parameter and the best method in each row are shown in bold.}
\label{tab:mask-f1}
\centering
\begin{small}
\begin{tabular}{llccccc}
\toprule
Setting & Value & LinADMG-SP & LiNGAM & BANG & DCD & RCD \\
\midrule
\multicolumn{7}{l}{\textit{Directed edges}} \\
\multicolumn{7}{l}{\emph{Dimension}} \\
$n$ & \textbf{5} & 0.746 $\pm$ 0.242 & \textbf{0.980 $\pm$ 0.052} & 0.000 $\pm$ 0.000 & 0.156 $\pm$ 0.139 & 0.350 $\pm$ 0.177 \\
$n$ & \textbf{10} & 0.724 $\pm$ 0.224 & \textbf{0.838 $\pm$ 0.051} & 0.010 $\pm$ 0.029 & 0.076 $\pm$ 0.078 & 0.427 $\pm$ 0.179 \\
$n$ & \textbf{20} & 0.557 $\pm$ 0.149 & \textbf{0.750 $\pm$ 0.072} & 0.000 $\pm$ 0.000 & 0.019 $\pm$ 0.023 & 0.532 $\pm$ 0.112 \\
\addlinespace
\multicolumn{7}{l}{\emph{Number of domains}} \\
$T/n$ & \textbf{1} & 0.724 $\pm$ 0.224 & \textbf{0.838 $\pm$ 0.051} & 0.010 $\pm$ 0.029 & 0.076 $\pm$ 0.078 & 0.427 $\pm$ 0.179 \\
$T/n$ & \textbf{2} & 0.690 $\pm$ 0.294 & \textbf{0.841 $\pm$ 0.041} & 0.005 $\pm$ 0.021 & 0.058 $\pm$ 0.059 & 0.470 $\pm$ 0.162 \\
$T/n$ & \textbf{3} & 0.784 $\pm$ 0.179 & \textbf{0.837 $\pm$ 0.044} & 0.005 $\pm$ 0.020 & 0.070 $\pm$ 0.064 & 0.565 $\pm$ 0.175 \\
\addlinespace
\multicolumn{7}{l}{\emph{Sample size}} \\
$N$ & \textbf{1000} & 0.679 $\pm$ 0.224 & \textbf{0.847 $\pm$ 0.062} & 0.000 $\pm$ 0.000 & 0.071 $\pm$ 0.081 & 0.386 $\pm$ 0.174 \\
$N$ & \textbf{5000} & 0.724 $\pm$ 0.224 & \textbf{0.838 $\pm$ 0.051} & 0.010 $\pm$ 0.029 & 0.076 $\pm$ 0.078 & 0.427 $\pm$ 0.179 \\
$N$ & \textbf{50000} & 0.729 $\pm$ 0.213 & \textbf{0.821 $\pm$ 0.058} & 0.005 $\pm$ 0.021 & 0.065 $\pm$ 0.088 & 0.373 $\pm$ 0.161 \\
\addlinespace
\multicolumn{7}{l}{\emph{Graph density}} \\
$k_X$ & \textbf{1.0} & 0.573 $\pm$ 0.194 & \textbf{0.702 $\pm$ 0.087} & 0.009 $\pm$ 0.041 & 0.093 $\pm$ 0.102 & 0.479 $\pm$ 0.206 \\
$k_X$ & \textbf{2.0} & 0.724 $\pm$ 0.224 & \textbf{0.838 $\pm$ 0.051} & 0.010 $\pm$ 0.029 & 0.076 $\pm$ 0.078 & 0.427 $\pm$ 0.179 \\
$k_X$ & \textbf{3.0} & 0.602 $\pm$ 0.241 & \textbf{0.923 $\pm$ 0.031} & 0.003 $\pm$ 0.014 & 0.072 $\pm$ 0.062 & 0.347 $\pm$ 0.178 \\
\addlinespace
\multicolumn{7}{l}{\emph{Confounding density}} \\
$k_U$ & \textbf{0.0} & 0.671 $\pm$ 0.223 & \textbf{1.000 $\pm$ 0.000} & 0.000 $\pm$ 0.000 & 0.100 $\pm$ 0.073 & 0.422 $\pm$ 0.151 \\
$k_U$ & \textbf{0.5} & 0.724 $\pm$ 0.224 & \textbf{0.838 $\pm$ 0.051} & 0.010 $\pm$ 0.029 & 0.076 $\pm$ 0.078 & 0.427 $\pm$ 0.179 \\
$k_U$ & \textbf{1.0} & 0.549 $\pm$ 0.262 & \textbf{0.738 $\pm$ 0.045} & 0.000 $\pm$ 0.000 & 0.057 $\pm$ 0.065 & 0.402 $\pm$ 0.189 \\
\addlinespace
\multicolumn{7}{l}{\emph{Noise distribution}} \\
Noise & \textbf{gauss} & 0.724 $\pm$ 0.224 & \textbf{0.838 $\pm$ 0.051} & 0.010 $\pm$ 0.029 & 0.076 $\pm$ 0.078 & 0.427 $\pm$ 0.179 \\
Noise & \textbf{exp} & 0.761 $\pm$ 0.182 & 0.837 $\pm$ 0.045 & \textbf{0.849 $\pm$ 0.071} & 0.063 $\pm$ 0.059 & 0.610 $\pm$ 0.142 \\
Noise & \textbf{gumbel} & \textbf{0.792 $\pm$ 0.138} & 0.702 $\pm$ 0.039 & 0.343 $\pm$ 0.143 & 0.071 $\pm$ 0.061 & 0.624 $\pm$ 0.138 \\
\midrule
\multicolumn{7}{l}{\textit{Bidirected edges}} \\
\multicolumn{7}{l}{\emph{Dimension}} \\
$n$ & \textbf{5} & \textbf{0.397 $\pm$ 0.132} & -- & 0.000 $\pm$ 0.000 & 0.125 $\pm$ 0.259 & 0.330 $\pm$ 0.156 \\
$n$ & \textbf{10} & \textbf{0.383 $\pm$ 0.109} & -- & 0.000 $\pm$ 0.000 & 0.022 $\pm$ 0.070 & 0.255 $\pm$ 0.110 \\
$n$ & \textbf{20} & \textbf{0.236 $\pm$ 0.066} & -- & 0.000 $\pm$ 0.000 & 0.010 $\pm$ 0.030 & 0.184 $\pm$ 0.082 \\
\addlinespace
\multicolumn{7}{l}{\emph{Number of domains}} \\
$T/n$ & \textbf{1} & \textbf{0.383 $\pm$ 0.109} & -- & 0.000 $\pm$ 0.000 & 0.022 $\pm$ 0.070 & 0.255 $\pm$ 0.110 \\
$T/n$ & \textbf{2} & \textbf{0.395 $\pm$ 0.111} & -- & 0.011 $\pm$ 0.050 & 0.010 $\pm$ 0.045 & 0.273 $\pm$ 0.138 \\
$T/n$ & \textbf{3} & \textbf{0.391 $\pm$ 0.093} & -- & 0.000 $\pm$ 0.000 & 0.010 $\pm$ 0.045 & 0.254 $\pm$ 0.099 \\
\addlinespace
\multicolumn{7}{l}{\emph{Sample size}} \\
$N$ & \textbf{1000} & \textbf{0.337 $\pm$ 0.093} & -- & 0.000 $\pm$ 0.000 & 0.012 $\pm$ 0.056 & 0.271 $\pm$ 0.103 \\
$N$ & \textbf{5000} & \textbf{0.383 $\pm$ 0.109} & -- & 0.000 $\pm$ 0.000 & 0.022 $\pm$ 0.070 & 0.255 $\pm$ 0.110 \\
$N$ & \textbf{50000} & \textbf{0.366 $\pm$ 0.116} & -- & 0.000 $\pm$ 0.000 & 0.012 $\pm$ 0.056 & 0.270 $\pm$ 0.127 \\
\addlinespace
\multicolumn{7}{l}{\emph{Graph density}} \\
$k_X$ & \textbf{1.0} & \textbf{0.421 $\pm$ 0.113} & -- & 0.000 $\pm$ 0.000 & 0.000 $\pm$ 0.000 & 0.332 $\pm$ 0.148 \\
$k_X$ & \textbf{2.0} & \textbf{0.383 $\pm$ 0.109} & -- & 0.000 $\pm$ 0.000 & 0.022 $\pm$ 0.070 & 0.255 $\pm$ 0.110 \\
$k_X$ & \textbf{3.0} & \textbf{0.298 $\pm$ 0.066} & -- & 0.000 $\pm$ 0.000 & 0.054 $\pm$ 0.096 & 0.285 $\pm$ 0.111 \\
\addlinespace
\multicolumn{7}{l}{\emph{Confounding density}} \\
$k_U$ & \textbf{0.0} & \textbf{0.000 $\pm$ 0.000} & -- & \textbf{0.000 $\pm$ 0.000} & \textbf{0.000 $\pm$ 0.000} & \textbf{0.000 $\pm$ 0.000} \\
$k_U$ & \textbf{0.5} & \textbf{0.383 $\pm$ 0.109} & -- & 0.000 $\pm$ 0.000 & 0.022 $\pm$ 0.070 & 0.255 $\pm$ 0.110 \\
$k_U$ & \textbf{1.0} & \textbf{0.618 $\pm$ 0.100} & -- & 0.010 $\pm$ 0.030 & 0.039 $\pm$ 0.057 & 0.441 $\pm$ 0.136 \\
\addlinespace
\multicolumn{7}{l}{\emph{Noise distribution}} \\
Noise & \textbf{gauss} & \textbf{0.383 $\pm$ 0.109} & -- & 0.000 $\pm$ 0.000 & 0.022 $\pm$ 0.070 & 0.255 $\pm$ 0.110 \\
Noise & \textbf{exp} & 0.395 $\pm$ 0.101 & -- & \textbf{0.726 $\pm$ 0.183} & 0.010 $\pm$ 0.045 & 0.274 $\pm$ 0.169 \\
Noise & \textbf{gumbel} & \textbf{0.387 $\pm$ 0.073} & -- & 0.328 $\pm$ 0.084 & 0.139 $\pm$ 0.140 & 0.278 $\pm$ 0.149 \\
\bottomrule
\end{tabular}
\end{small}
\end{table*}

\begin{table*}[t]
\caption{UP-Observations results: average SHD $\pm$ standard deviation over 20 random seeds. The varying parameter and the best method in each row are shown in bold.}
\label{tab:mask-shd}
\centering
\begin{small}
\begin{tabular}{llccccc}
\toprule
Setting & Value & LinADMG-SP & LiNGAM & BANG & DCD & RCD \\
\midrule
\multicolumn{7}{l}{\textit{Directed edges}} \\
\multicolumn{7}{l}{\emph{Dimension}} \\
$n$ & \textbf{5} & 4.400 $\pm$ 4.309 & \textbf{0.400 $\pm$ 1.046} & 10.000 $\pm$ 0.000 & 9.800 $\pm$ 1.361 & 7.750 $\pm$ 1.251 \\
$n$ & \textbf{10} & 11.500 $\pm$ 8.930 & \textbf{7.750 $\pm$ 2.751} & 20.000 $\pm$ 0.459 & 20.600 $\pm$ 1.930 & 15.300 $\pm$ 3.373 \\
$n$ & \textbf{20} & 39.650 $\pm$ 12.032 & 27.650 $\pm$ 12.398 & 40.200 $\pm$ 0.410 & 41.450 $\pm$ 1.395 & \textbf{27.400 $\pm$ 4.358} \\
\addlinespace
\multicolumn{7}{l}{\emph{Number of domains}} \\
$T/n$ & \textbf{1} & 11.500 $\pm$ 8.930 & \textbf{7.750 $\pm$ 2.751} & 20.000 $\pm$ 0.459 & 20.600 $\pm$ 1.930 & 15.300 $\pm$ 3.373 \\
$T/n$ & \textbf{2} & 13.800 $\pm$ 13.836 & \textbf{7.650 $\pm$ 2.390} & 20.000 $\pm$ 0.324 & 20.600 $\pm$ 1.188 & 14.500 $\pm$ 3.086 \\
$T/n$ & \textbf{3} & 9.500 $\pm$ 8.513 & \textbf{7.900 $\pm$ 2.532} & 20.250 $\pm$ 0.716 & 20.500 $\pm$ 1.100 & 12.650 $\pm$ 3.924 \\
\addlinespace
\multicolumn{7}{l}{\emph{Sample size}} \\
$N$ & \textbf{1000} & 13.850 $\pm$ 9.869 & \textbf{7.300 $\pm$ 3.358} & 20.000 $\pm$ 0.000 & 20.400 $\pm$ 1.536 & 16.250 $\pm$ 3.492 \\
$N$ & \textbf{5000} & 11.500 $\pm$ 8.930 & \textbf{7.750 $\pm$ 2.751} & 20.000 $\pm$ 0.459 & 20.600 $\pm$ 1.930 & 15.300 $\pm$ 3.373 \\
$N$ & \textbf{50000} & 11.600 $\pm$ 9.355 & \textbf{8.800 $\pm$ 3.286} & 19.950 $\pm$ 0.224 & 20.650 $\pm$ 1.631 & 16.100 $\pm$ 2.404 \\
\addlinespace
\multicolumn{7}{l}{\emph{Graph density}} \\
$k_X$ & \textbf{1.0} & 12.300 $\pm$ 5.110 & 8.850 $\pm$ 3.558 & 10.050 $\pm$ 0.394 & 10.450 $\pm$ 1.234 & \textbf{7.400 $\pm$ 1.875} \\
$k_X$ & \textbf{2.0} & 11.500 $\pm$ 8.930 & \textbf{7.750 $\pm$ 2.751} & 20.000 $\pm$ 0.459 & 20.600 $\pm$ 1.930 & 15.300 $\pm$ 3.373 \\
$k_X$ & \textbf{3.0} & 24.450 $\pm$ 16.504 & \textbf{5.050 $\pm$ 2.139} & 30.100 $\pm$ 0.308 & 30.250 $\pm$ 1.832 & 23.650 $\pm$ 4.380 \\
\addlinespace
\multicolumn{7}{l}{\emph{Confounding density}} \\
$k_U$ & \textbf{0.0} & 14.600 $\pm$ 10.252 & \textbf{0.000 $\pm$ 0.000} & 20.000 $\pm$ 0.000 & 20.200 $\pm$ 1.576 & 14.800 $\pm$ 2.687 \\
$k_U$ & \textbf{0.5} & 11.500 $\pm$ 8.930 & \textbf{7.750 $\pm$ 2.751} & 20.000 $\pm$ 0.459 & 20.600 $\pm$ 1.930 & 15.300 $\pm$ 3.373 \\
$k_U$ & \textbf{1.0} & 21.250 $\pm$ 13.325 & \textbf{14.200 $\pm$ 3.071} & 20.000 $\pm$ 0.000 & 20.450 $\pm$ 0.887 & 15.500 $\pm$ 3.317 \\
\addlinespace
\multicolumn{7}{l}{\emph{Noise distribution}} \\
Noise & \textbf{gauss} & 11.500 $\pm$ 8.930 & \textbf{7.750 $\pm$ 2.751} & 20.000 $\pm$ 0.459 & 20.600 $\pm$ 1.930 & 15.300 $\pm$ 3.373 \\
Noise & \textbf{exp} & 10.050 $\pm$ 7.222 & 7.850 $\pm$ 2.540 & \textbf{6.050 $\pm$ 2.856} & 20.650 $\pm$ 1.387 & 12.100 $\pm$ 3.712 \\
Noise & \textbf{gumbel} & \textbf{8.950 $\pm$ 5.698} & 17.050 $\pm$ 3.017 & 19.100 $\pm$ 2.292 & 20.500 $\pm$ 1.235 & 11.700 $\pm$ 3.294 \\
\midrule
\multicolumn{7}{l}{\textit{Bidirected edges}} \\
\multicolumn{7}{l}{\emph{Dimension}} \\
$n$ & \textbf{5} & 5.800 $\pm$ 1.542 & -- & 2.250 $\pm$ 0.444 & \textbf{2.150 $\pm$ 0.671} & 5.550 $\pm$ 1.761 \\
$n$ & \textbf{10} & 21.150 $\pm$ 4.416 & -- & \textbf{7.300 $\pm$ 1.261} & 7.750 $\pm$ 1.372 & 16.600 $\pm$ 3.733 \\
$n$ & \textbf{20} & 103.250 $\pm$ 39.034 & -- & 16.800 $\pm$ 4.200 & \textbf{16.700 $\pm$ 3.600} & 38.000 $\pm$ 9.414 \\
\addlinespace
\multicolumn{7}{l}{\emph{Number of domains}} \\
$T/n$ & \textbf{1} & 21.150 $\pm$ 4.416 & -- & \textbf{7.300 $\pm$ 1.261} & 7.750 $\pm$ 1.372 & 16.600 $\pm$ 3.733 \\
$T/n$ & \textbf{2} & 20.150 $\pm$ 5.566 & -- & \textbf{7.250 $\pm$ 1.251} & 7.650 $\pm$ 1.089 & 15.500 $\pm$ 4.310 \\
$T/n$ & \textbf{3} & 20.100 $\pm$ 4.667 & -- & \textbf{7.250 $\pm$ 1.293} & 7.600 $\pm$ 1.231 & 14.150 $\pm$ 3.815 \\
\addlinespace
\multicolumn{7}{l}{\emph{Sample size}} \\
$N$ & \textbf{1000} & 23.550 $\pm$ 4.639 & -- & \textbf{7.250 $\pm$ 1.293} & 7.650 $\pm$ 1.309 & 16.950 $\pm$ 3.486 \\
$N$ & \textbf{5000} & 21.150 $\pm$ 4.416 & -- & \textbf{7.300 $\pm$ 1.261} & 7.750 $\pm$ 1.372 & 16.600 $\pm$ 3.733 \\
$N$ & \textbf{50000} & 21.600 $\pm$ 5.205 & -- & \textbf{7.300 $\pm$ 1.302} & 7.850 $\pm$ 1.424 & 18.200 $\pm$ 5.227 \\
\addlinespace
\multicolumn{7}{l}{\emph{Graph density}} \\
$k_X$ & \textbf{1.0} & 18.000 $\pm$ 5.161 & -- & \textbf{7.250 $\pm$ 1.293} & 7.500 $\pm$ 1.318 & 10.100 $\pm$ 2.972 \\
$k_X$ & \textbf{2.0} & 21.150 $\pm$ 4.416 & -- & \textbf{7.300 $\pm$ 1.261} & 7.750 $\pm$ 1.372 & 16.600 $\pm$ 3.733 \\
$k_X$ & \textbf{3.0} & 26.550 $\pm$ 3.913 & -- & \textbf{7.250 $\pm$ 1.293} & 7.750 $\pm$ 1.333 & 20.100 $\pm$ 4.941 \\
\addlinespace
\multicolumn{7}{l}{\emph{Confounding density}} \\
$k_U$ & \textbf{0.0} & 24.000 $\pm$ 2.753 & -- & \textbf{0.050 $\pm$ 0.224} & 0.200 $\pm$ 0.410 & 16.100 $\pm$ 4.844 \\
$k_U$ & \textbf{0.5} & 21.150 $\pm$ 4.416 & -- & \textbf{7.300 $\pm$ 1.261} & 7.750 $\pm$ 1.372 & 16.600 $\pm$ 3.733 \\
$k_U$ & \textbf{1.0} & \textbf{19.900 $\pm$ 4.128} & -- & 21.150 $\pm$ 4.545 & 21.050 $\pm$ 4.186 & 21.000 $\pm$ 4.104 \\
\addlinespace
\multicolumn{7}{l}{\emph{Noise distribution}} \\
Noise & \textbf{gauss} & 21.150 $\pm$ 4.416 & -- & \textbf{7.300 $\pm$ 1.261} & 7.750 $\pm$ 1.372 & 16.600 $\pm$ 3.733 \\
Noise & \textbf{exp} & 19.150 $\pm$ 3.787 & -- & \textbf{3.300 $\pm$ 2.179} & 7.750 $\pm$ 1.164 & 13.050 $\pm$ 4.513 \\
Noise & \textbf{gumbel} & 20.950 $\pm$ 3.605 & -- & 25.350 $\pm$ 5.143 & \textbf{9.250 $\pm$ 1.860} & 12.450 $\pm$ 2.911 \\
\bottomrule
\end{tabular}
\end{small}
\end{table*}

\begin{table*}[t]
\caption{UMN-Hard Interventions results: average F1 $\pm$ standard deviation over 20 random seeds. The varying parameter and the best method in each row are shown in bold.}
\label{tab:hard-intervention-f1}
\centering
\begin{small}
\begin{tabular}{llccccc}
\toprule
Setting & Value & LinADMG-UMNI & LiNGAM & BANG & DCD & RCD \\
\midrule
\multicolumn{7}{l}{\textit{Directed edges}} \\
\multicolumn{7}{l}{\emph{Dimension}} \\
$n$ & \textbf{5} & \textbf{0.994 $\pm$ 0.025} & 0.960 $\pm$ 0.082 & 0.000 $\pm$ 0.000 & 0.209 $\pm$ 0.126 & 0.000 $\pm$ 0.000 \\
$n$ & \textbf{10} & 0.601 $\pm$ 0.366 & \textbf{0.821 $\pm$ 0.069} & 0.000 $\pm$ 0.000 & 0.058 $\pm$ 0.071 & 0.000 $\pm$ 0.000 \\
$n$ & \textbf{20} & -- & \textbf{0.705 $\pm$ 0.083} & 0.002 $\pm$ 0.011 & 0.019 $\pm$ 0.028 & 0.000 $\pm$ 0.000 \\
\addlinespace
\multicolumn{7}{l}{\emph{Number of domains}} \\
$T/n$ & \textbf{1} & 0.601 $\pm$ 0.366 & \textbf{0.821 $\pm$ 0.069} & 0.000 $\pm$ 0.000 & 0.058 $\pm$ 0.071 & 0.000 $\pm$ 0.000 \\
$T/n$ & \textbf{2} & 0.377 $\pm$ 0.161 & \textbf{0.801 $\pm$ 0.073} & 0.000 $\pm$ 0.000 & 0.069 $\pm$ 0.068 & 0.000 $\pm$ 0.000 \\
$T/n$ & \textbf{3} & -- & \textbf{0.814 $\pm$ 0.072} & 0.000 $\pm$ 0.000 & 0.071 $\pm$ 0.070 & 0.000 $\pm$ 0.000 \\
\addlinespace
\multicolumn{7}{l}{\emph{Sample size}} \\
$N$ & \textbf{1000} & -- & \textbf{0.724 $\pm$ 0.138} & 0.005 $\pm$ 0.021 & 0.084 $\pm$ 0.072 & 0.000 $\pm$ 0.000 \\
$N$ & \textbf{5000} & -- & \textbf{0.823 $\pm$ 0.069} & 0.000 $\pm$ 0.000 & 0.066 $\pm$ 0.073 & 0.000 $\pm$ 0.000 \\
$N$ & \textbf{50000} & 0.601 $\pm$ 0.366 & \textbf{0.821 $\pm$ 0.069} & 0.000 $\pm$ 0.000 & 0.058 $\pm$ 0.071 & 0.000 $\pm$ 0.000 \\
\addlinespace
\multicolumn{7}{l}{\emph{Graph density}} \\
$k_X$ & \textbf{1.0} & 0.576 $\pm$ 0.403 & \textbf{0.656 $\pm$ 0.116} & 0.009 $\pm$ 0.041 & 0.049 $\pm$ 0.094 & 0.000 $\pm$ 0.000 \\
$k_X$ & \textbf{2.0} & 0.601 $\pm$ 0.366 & \textbf{0.821 $\pm$ 0.069} & 0.000 $\pm$ 0.000 & 0.058 $\pm$ 0.071 & 0.000 $\pm$ 0.000 \\
$k_X$ & \textbf{3.0} & 0.720 $\pm$ 0.286 & \textbf{0.920 $\pm$ 0.033} & 0.000 $\pm$ 0.000 & 0.076 $\pm$ 0.057 & 0.000 $\pm$ 0.000 \\
\addlinespace
\multicolumn{7}{l}{\emph{Confounding density}} \\
$k_U$ & \textbf{0.0} & 0.574 $\pm$ 0.338 & \textbf{1.000 $\pm$ 0.000} & 0.000 $\pm$ 0.000 & 0.068 $\pm$ 0.050 & 0.000 $\pm$ 0.000 \\
$k_U$ & \textbf{0.5} & 0.601 $\pm$ 0.366 & \textbf{0.821 $\pm$ 0.069} & 0.000 $\pm$ 0.000 & 0.058 $\pm$ 0.071 & 0.000 $\pm$ 0.000 \\
$k_U$ & \textbf{1.0} & 0.519 $\pm$ 0.367 & \textbf{0.677 $\pm$ 0.086} & 0.000 $\pm$ 0.000 & 0.085 $\pm$ 0.066 & 0.000 $\pm$ 0.000 \\
\addlinespace
\multicolumn{7}{l}{\emph{Noise distribution}} \\
Noise & \textbf{gauss} & 0.601 $\pm$ 0.366 & \textbf{0.821 $\pm$ 0.069} & 0.000 $\pm$ 0.000 & 0.058 $\pm$ 0.071 & 0.000 $\pm$ 0.000 \\
Noise & \textbf{exp} & 0.618 $\pm$ 0.280 & 0.808 $\pm$ 0.068 & \textbf{0.842 $\pm$ 0.077} & 0.085 $\pm$ 0.078 & 0.753 $\pm$ 0.118 \\
Noise & \textbf{gumbel} & 0.723 $\pm$ 0.328 & 0.805 $\pm$ 0.068 & \textbf{0.826 $\pm$ 0.073} & 0.066 $\pm$ 0.074 & 0.442 $\pm$ 0.177 \\
\midrule
\multicolumn{7}{l}{\textit{Bidirected edges}} \\
\multicolumn{7}{l}{\emph{Dimension}} \\
$n$ & \textbf{5} & \textbf{0.969 $\pm$ 0.104} & --& 0.000 $\pm$ 0.000 & 0.238 $\pm$ 0.286 & 0.332 $\pm$ 0.116 \\
$n$ & \textbf{10} & \textbf{0.578 $\pm$ 0.259} & -- & 0.000 $\pm$ 0.000 & 0.192 $\pm$ 0.160 & 0.311 $\pm$ 0.079 \\
$n$ & \textbf{20} & --& -- & 0.000 $\pm$ 0.000 & 0.158 $\pm$ 0.121 & \textbf{0.263 $\pm$ 0.055} \\
\addlinespace
\multicolumn{7}{l}{\emph{Number of domains}} \\
$T/n$ & \textbf{1} & \textbf{0.578 $\pm$ 0.259} & -- & 0.000 $\pm$ 0.000 & 0.192 $\pm$ 0.160 & 0.311 $\pm$ 0.079 \\
$T/n$ & \textbf{2} & \textbf{0.315 $\pm$ 0.110} & -- & 0.000 $\pm$ 0.000 & 0.182 $\pm$ 0.120 & 0.308 $\pm$ 0.091 \\
$T/n$ & \textbf{3} & --& -- & 0.000 $\pm$ 0.000 & 0.166 $\pm$ 0.127 & \textbf{0.312 $\pm$ 0.076} \\
\addlinespace
\multicolumn{7}{l}{\emph{Sample size}} \\
$N$ & \textbf{1000} & -- & -- & 0.000 $\pm$ 0.000 & 0.194 $\pm$ 0.134 & \textbf{0.297 $\pm$ 0.092} \\
$N$ & \textbf{5000} & -- & -- & 0.000 $\pm$ 0.000 & 0.182 $\pm$ 0.140 & \textbf{0.318 $\pm$ 0.088} \\
$N$ & \textbf{50000} & \textbf{0.578 $\pm$ 0.259} & -- & 0.000 $\pm$ 0.000 & 0.192 $\pm$ 0.160 & 0.311 $\pm$ 0.079 \\
\addlinespace
\multicolumn{7}{l}{\emph{Graph density}} \\
$k_X$ & \textbf{1.0} & \textbf{0.672 $\pm$ 0.252} & -- & 0.000 $\pm$ 0.000 & 0.086 $\pm$ 0.113 & 0.400 $\pm$ 0.089 \\
$k_X$ & \textbf{2.0} & \textbf{0.578 $\pm$ 0.259} & -- & 0.000 $\pm$ 0.000 & 0.192 $\pm$ 0.160 & 0.311 $\pm$ 0.079 \\
$k_X$ & \textbf{3.0} & \textbf{0.556 $\pm$ 0.253} & -- & 0.012 $\pm$ 0.056 & 0.173 $\pm$ 0.152 & 0.257 $\pm$ 0.073 \\
\addlinespace
\multicolumn{7}{l}{\emph{Confounding density}} \\
$k_U$ & \textbf{0.0} & \textbf{0.000 $\pm$ 0.000} & -- & \textbf{0.000 $\pm$ 0.000} & \textbf{0.000 $\pm$ 0.000} & \textbf{0.000 $\pm$ 0.000} \\
$k_U$ & \textbf{0.5} & \textbf{0.578 $\pm$ 0.259} & --& 0.000 $\pm$ 0.000 & 0.192 $\pm$ 0.160 & 0.311 $\pm$ 0.079 \\
$k_U$ & \textbf{1.0} & \textbf{0.678 $\pm$ 0.165} & --& 0.000 $\pm$ 0.000 & 0.171 $\pm$ 0.104 & 0.587 $\pm$ 0.078 \\
\addlinespace
\multicolumn{7}{l}{\emph{Noise distribution}} \\
Noise & \textbf{gauss} & \textbf{0.578 $\pm$ 0.259} & -- & 0.000 $\pm$ 0.000 & 0.192 $\pm$ 0.160 & 0.311 $\pm$ 0.079 \\
Noise & \textbf{exp} & 0.470 $\pm$ 0.179 & --& \textbf{0.779 $\pm$ 0.189} & 0.147 $\pm$ 0.112 & 0.205 $\pm$ 0.171 \\
Noise & \textbf{gumbel} & 0.670 $\pm$ 0.235 & -- & \textbf{0.763 $\pm$ 0.211} & 0.276 $\pm$ 0.125 & 0.305 $\pm$ 0.108 \\
\bottomrule
\end{tabular}
\end{small}
\end{table*}

\begin{table*}[t]
\caption{UMN-Hard Interventions results: average SHD $\pm$ standard deviation over 20 random seeds. The varying parameter and the best method in each row are shown in bold.}
\label{tab:hard-intervention-shd}
\centering
\begin{small}
\begin{tabular}{llccccc}
\toprule
Setting & Value & LinADMG-UMNI & LiNGAM & BANG & DCD & RCD \\
\midrule
\multicolumn{7}{l}{\textit{Directed edges}} \\
\multicolumn{7}{l}{\emph{Dimension}} \\
$n$ & \textbf{5} & \textbf{0.100 $\pm$ 0.447} & 0.800 $\pm$ 1.642 & 10.000 $\pm$ 0.000 & 9.450 $\pm$ 1.146 & 10.000 $\pm$ 0.000 \\
$n$ & \textbf{10} & 12.200 $\pm$ 9.350 & \textbf{8.850 $\pm$ 4.069} & 20.100 $\pm$ 0.308 & 20.650 $\pm$ 1.663 & 20.050 $\pm$ 0.224 \\
$n$ & \textbf{20} & -- & \textbf{34.250 $\pm$ 14.407} & 40.000 $\pm$ 0.324 & 41.050 $\pm$ 1.731 & 40.000 $\pm$ 0.000 \\
\addlinespace
\multicolumn{7}{l}{\emph{Number of domains}} \\
$T/n$ & \textbf{1} & 12.200 $\pm$ 9.350 & \textbf{8.850 $\pm$ 4.069} & 20.100 $\pm$ 0.308 & 20.650 $\pm$ 1.663 & 20.050 $\pm$ 0.224 \\
$T/n$ & \textbf{2} & 23.556 $\pm$ 6.070 & \textbf{9.900 $\pm$ 4.436} & 20.100 $\pm$ 0.308 & 20.550 $\pm$ 1.356 & 20.000 $\pm$ 0.000 \\
$T/n$ & \textbf{3} & -- & \textbf{9.300 $\pm$ 4.566} & 20.050 $\pm$ 0.224 & 20.400 $\pm$ 1.273 & 20.000 $\pm$ 0.000 \\
\addlinespace
\multicolumn{7}{l}{\emph{Sample size}} \\
$N$ & \textbf{1000} & -- & \textbf{14.150 $\pm$ 8.487} & 20.000 $\pm$ 0.324 & 20.350 $\pm$ 1.755 & 20.000 $\pm$ 0.000 \\
$N$ & \textbf{5000} & -- & \textbf{8.650 $\pm$ 4.030} & 20.000 $\pm$ 0.000 & 20.450 $\pm$ 1.276 & 20.000 $\pm$ 0.000 \\
$N$ & \textbf{50000} & 12.200 $\pm$ 9.350 & \textbf{8.850 $\pm$ 4.069} & 20.100 $\pm$ 0.308 & 20.650 $\pm$ 1.663 & 20.050 $\pm$ 0.224 \\
\addlinespace
\multicolumn{7}{l}{\emph{Graph density}} \\
$k_X$ & \textbf{1.0} & \textbf{6.800 $\pm$ 5.307} & 11.000 $\pm$ 5.400 & 9.950 $\pm$ 0.224 & 10.950 $\pm$ 1.050 & 10.000 $\pm$ 0.000 \\
$k_X$ & \textbf{2.0} & 12.200 $\pm$ 9.350 & \textbf{8.850 $\pm$ 4.069} & 20.100 $\pm$ 0.308 & 20.650 $\pm$ 1.663 & 20.050 $\pm$ 0.224 \\
$k_X$ & \textbf{3.0} & 12.850 $\pm$ 10.752 & \textbf{5.250 $\pm$ 2.337} & 30.050 $\pm$ 0.224 & 29.850 $\pm$ 1.424 & 30.000 $\pm$ 0.000 \\
\addlinespace
\multicolumn{7}{l}{\emph{Confounding density}} \\
$k_U$ & \textbf{0.0} & 12.150 $\pm$ 7.748 & \textbf{0.000 $\pm$ 0.000} & 20.050 $\pm$ 0.224 & 20.500 $\pm$ 1.357 & 20.000 $\pm$ 0.000 \\
$k_U$ & \textbf{0.5} & 12.200 $\pm$ 9.350 & \textbf{8.850 $\pm$ 4.069} & 20.100 $\pm$ 0.308 & 20.650 $\pm$ 1.663 & 20.050 $\pm$ 0.224 \\
$k_U$ & \textbf{1.0} & \textbf{14.200 $\pm$ 9.328} & 18.300 $\pm$ 5.904 & 20.050 $\pm$ 0.224 & 20.250 $\pm$ 1.410 & 20.000 $\pm$ 0.000 \\
\addlinespace
\multicolumn{7}{l}{\emph{Noise distribution}} \\
Noise & \textbf{gauss} & 12.200 $\pm$ 9.350 & \textbf{8.850 $\pm$ 4.069} & 20.100 $\pm$ 0.308 & 20.650 $\pm$ 1.663 & 20.050 $\pm$ 0.224 \\
Noise & \textbf{exp} & 12.950 $\pm$ 7.015 & 9.250 $\pm$ 3.552 & \textbf{6.400 $\pm$ 3.152} & 20.100 $\pm$ 1.553 & 8.750 $\pm$ 3.640 \\
Noise & \textbf{gumbel} & 8.550 $\pm$ 8.319 & 9.450 $\pm$ 3.531 & \textbf{7.100 $\pm$ 3.042} & 20.500 $\pm$ 1.357 & 15.050 $\pm$ 3.561 \\
\midrule
\multicolumn{7}{l}{\textit{Bidirected edges}} \\
\multicolumn{7}{l}{\emph{Dimension}} \\
$n$ & \textbf{5} & \textbf{0.200 $\pm$ 0.696} & -- & 2.250 $\pm$ 0.444 & 2.750 $\pm$ 1.446 & 6.750 $\pm$ 1.293 \\
$n$ & \textbf{10} & 10.350 $\pm$ 7.043 & -- & \textbf{7.250 $\pm$ 1.293} & 8.500 $\pm$ 1.960 & 23.200 $\pm$ 3.238 \\
$n$ & \textbf{20} & -- & -- & \textbf{16.500 $\pm$ 3.940} & 19.000 $\pm$ 4.531 & 60.850 $\pm$ 8.273 \\
\addlinespace
\multicolumn{7}{l}{\emph{Number of domains}} \\
$T/n$ & \textbf{1} & 10.350 $\pm$ 7.043 & -- & \textbf{7.250 $\pm$ 1.293} & 8.500 $\pm$ 1.960 & 23.200 $\pm$ 3.238 \\
$T/n$ & \textbf{2} & 21.167 $\pm$ 5.021 & -- & \textbf{7.250 $\pm$ 1.293} & 8.700 $\pm$ 1.720 & 23.150 $\pm$ 4.030 \\
$T/n$ & \textbf{3} & -- & -- & \textbf{7.250 $\pm$ 1.293} & 8.750 $\pm$ 1.803 & 22.750 $\pm$ 2.936 \\
\addlinespace
\multicolumn{7}{l}{\emph{Sample size}} \\
$N$ & \textbf{1000} & -- & -- & \textbf{7.250 $\pm$ 1.293} & 8.400 $\pm$ 1.729 & 23.650 $\pm$ 4.056 \\
$N$ & \textbf{5000} & -- & -- & \textbf{7.250 $\pm$ 1.293} & 8.550 $\pm$ 2.064 & 23.100 $\pm$ 3.597 \\
$N$ & \textbf{50000} & 10.350 $\pm$ 7.043 & -- & \textbf{7.250 $\pm$ 1.293} & 8.500 $\pm$ 1.960 & 23.200 $\pm$ 3.238 \\
\addlinespace
\multicolumn{7}{l}{\emph{Graph density}} \\
$k_X$ & \textbf{1.0} & \textbf{6.950 $\pm$ 5.916} & -- & 7.250 $\pm$ 1.293 & 8.500 $\pm$ 1.762 & 14.300 $\pm$ 4.105 \\
$k_X$ & \textbf{2.0} & 10.350 $\pm$ 7.043 & -- & \textbf{7.250 $\pm$ 1.293} & 8.500 $\pm$ 1.960 & 23.200 $\pm$ 3.238 \\
$k_X$ & \textbf{3.0} & 10.400 $\pm$ 6.402 & -- & \textbf{7.200 $\pm$ 1.322} & 9.900 $\pm$ 1.683 & 28.900 $\pm$ 3.076 \\
\addlinespace
\multicolumn{7}{l}{\emph{Confounding density}} \\
$k_U$ & \textbf{0.0} & 10.800 $\pm$ 6.910 & -- & \textbf{0.000 $\pm$ 0.000} & 2.800 $\pm$ 1.704 & 24.650 $\pm$ 3.631 \\
$k_U$ & \textbf{0.5} & 10.350 $\pm$ 7.043 & -- & \textbf{7.250 $\pm$ 1.293} & 8.500 $\pm$ 1.960 & 23.200 $\pm$ 3.238 \\
$k_U$ & \textbf{1.0} & \textbf{13.750 $\pm$ 7.440} & -- & 21.200 $\pm$ 4.514 & 21.050 $\pm$ 3.953 & 20.000 $\pm$ 2.974 \\
\addlinespace
\multicolumn{7}{l}{\emph{Noise distribution}} \\
Noise & \textbf{gauss} & 10.350 $\pm$ 7.043 & -- & \textbf{7.250 $\pm$ 1.293} & 8.500 $\pm$ 1.960 & 23.200 $\pm$ 3.238 \\
Noise & \textbf{exp} & 13.000 $\pm$ 4.600 & -- & \textbf{2.800 $\pm$ 2.353} & 9.050 $\pm$ 1.905 & 9.550 $\pm$ 2.819 \\
Noise & \textbf{gumbel} & 7.650 $\pm$ 6.831 & -- & \textbf{3.150 $\pm$ 2.720} & 11.400 $\pm$ 3.560 & 15.500 $\pm$ 3.269 \\
\bottomrule
\end{tabular}
\end{small}
\end{table*}

\FloatBarrier

\subsection{VaDE vs. diagGMM likelihood loss}
\label{app: vade}
In this section, we compare \texttt{LinADMG-diagGMM}  and VaDE, as in \citet{kivva2022identifiability}, under the DiagGMMs data-generating process. Starting from the anchor configuration $n=10, N=5000, k_X=2$, it varies the number of variables $n\in\{5,10,20\}$, sample size $N\in\{1000,5000,50000\}$, and graph density $k_X\in\{1,2,3\}$. We evaluate each configuration over 20 random seeds and report performance separately for directed and bidirected edges using mean F1 and SHD, with standard deviations. \texttt{LinADMG-diagGMM}  clearly outperforms VaDE in estimating directed edges across every configuration, achieving significantly higher F1 and lower SHD. The performance of both methods decreases as dimension increases, whereas increasing sample size has relatively little effect. Higher graph density improves directed F1 but generally increases SHD. For bidirected edges, both methods perform well, and their F1 scores are much closer: VaDE is slightly better in several low- or moderate-complexity settings, while \texttt{LinADMG-diagGMM}  performs better for $n=20$, $N=1000$, $N=50000$, and $k_X=3$. 
\begin{table}[h]
\caption{Diag-GMMs results: average F1 and SHD $\pm$ standard deviation over 20 random seeds. The varying parameter and the best method for each metric in each row are shown in bold.}
\label{tab:auxillary-f1-shd}
\centering
\begin{small}
\begin{tabular}{llcccc}
\toprule
Setting & Value & \multicolumn{2}{c}{F1} & \multicolumn{2}{c}{SHD} \\
\cmidrule(lr){3-4}\cmidrule(lr){5-6}
 &  & LinADMG-diagGMM & VaDE-linear & LinADMG-diagGMM & VaDE-linear \\
\midrule
\multicolumn{6}{l}{\textit{Directed edges}} \\
\multicolumn{6}{l}{\emph{Dimension}} \\
$n$ & \textbf{5} & \textbf{0.873 $\pm$ 0.157} & 0.421 $\pm$ 0.175 & \textbf{2.200 $\pm$ 2.546} & 8.700 $\pm$ 2.342 \\
$n$ & \textbf{10} & \textbf{0.817 $\pm$ 0.100} & 0.254 $\pm$ 0.112 & \textbf{7.650 $\pm$ 3.897} & 26.050 $\pm$ 3.900 \\
$n$ & \textbf{20} & \textbf{0.743 $\pm$ 0.081} & 0.144 $\pm$ 0.043 & \textbf{24.800 $\pm$ 9.339} & 68.000 $\pm$ 4.702 \\
\addlinespace
\multicolumn{6}{l}{\emph{Sample size}} \\
$N$ & \textbf{1000} & \textbf{0.806 $\pm$ 0.092} & 0.218 $\pm$ 0.084 & \textbf{7.950 $\pm$ 3.634} & 26.950 $\pm$ 2.911 \\
$N$ & \textbf{5000} & \textbf{0.817 $\pm$ 0.100} & 0.254 $\pm$ 0.112 & \textbf{7.650 $\pm$ 3.897} & 26.050 $\pm$ 3.900 \\
$N$ & \textbf{50000} & \textbf{0.821 $\pm$ 0.076} & 0.243 $\pm$ 0.108 & \textbf{7.600 $\pm$ 3.331} & 26.600 $\pm$ 3.691 \\
\addlinespace
\multicolumn{6}{l}{\emph{Graph density}} \\
$k_X$ & \textbf{1.0} & \textbf{0.691 $\pm$ 0.107} & 0.178 $\pm$ 0.092 & \textbf{8.550 $\pm$ 3.832} & 20.150 $\pm$ 2.925 \\
$k_X$ & \textbf{2.0} & \textbf{0.817 $\pm$ 0.100} & 0.254 $\pm$ 0.112 & \textbf{7.650 $\pm$ 3.897} & 26.050 $\pm$ 3.900 \\
$k_X$ & \textbf{3.0} & \textbf{0.833 $\pm$ 0.109} & 0.269 $\pm$ 0.107 & \textbf{9.550 $\pm$ 5.698} & 32.650 $\pm$ 4.534 \\
\midrule
\multicolumn{6}{l}{\textit{Bidirected edges}} \\
\multicolumn{6}{l}{\emph{Dimension}} \\
$n$ & \textbf{5} & 0.956 $\pm$ 0.077 & \textbf{0.959 $\pm$ 0.052} & \textbf{0.750 $\pm$ 1.251} & \textbf{0.750 $\pm$ 0.910} \\
$n$ & \textbf{10} & 0.930 $\pm$ 0.198 & \textbf{0.939 $\pm$ 0.050} & \textbf{4.200 $\pm$ 9.876} & 5.000 $\pm$ 3.853 \\
$n$ & \textbf{20} & \textbf{0.951 $\pm$ 0.082} & 0.882 $\pm$ 0.061 & \textbf{16.050 $\pm$ 23.712} & 39.200 $\pm$ 17.981 \\
\addlinespace
\multicolumn{6}{l}{\emph{Sample size}} \\
$N$ & \textbf{1000} & \textbf{0.971 $\pm$ 0.049} & 0.940 $\pm$ 0.049 & \textbf{2.350 $\pm$ 3.951} & 4.900 $\pm$ 3.698 \\
$N$ & \textbf{5000} & 0.930 $\pm$ 0.198 & \textbf{0.939 $\pm$ 0.050} & \textbf{4.200 $\pm$ 9.876} & 5.000 $\pm$ 3.853 \\
$N$ & \textbf{50000} & \textbf{0.973 $\pm$ 0.064} & 0.937 $\pm$ 0.053 & \textbf{2.100 $\pm$ 4.767} & 5.150 $\pm$ 4.017 \\
\addlinespace
\multicolumn{6}{l}{\emph{Graph density}} \\
$k_X$ & \textbf{1.0} & 0.920 $\pm$ 0.151 & \textbf{0.934 $\pm$ 0.038} & \textbf{5.500 $\pm$ 8.593} & \textbf{5.500 $\pm$ 3.017} \\
$k_X$ & \textbf{2.0} & 0.930 $\pm$ 0.198 & \textbf{0.939 $\pm$ 0.050} & \textbf{4.200 $\pm$ 9.876} & 5.000 $\pm$ 3.853 \\
$k_X$ & \textbf{3.0} & \textbf{0.921 $\pm$ 0.137} & 0.918 $\pm$ 0.045 & \textbf{5.650 $\pm$ 8.061} & 6.650 $\pm$ 3.528 \\
\bottomrule
\end{tabular}
\end{small}
\end{table}

\subsection{MuDo-NLL vs. MuDo-Cov}
\label{app: mudocov}
In this section, we compare the training objective in MuDo-Cov used in \citet{xu2026identifying} with MuDo-NLL, which we use in this work. MuDoNLL minimizes the Gaussian negative log-likelihood, while MuDo-Cov minimizes the covariance matching error. We hold all other optimization and evaluation settings fixed, and evaluate every condition using the same 20 random seeds. Starting from the anchor configuration $n=10, N=5000,T=n,k_X=2,k_U=0.5$, and Gaussian noise, we vary one factor at a time. We evaluate performance separately for directed and bidirected edges using F1 (higher is better) and SHD (lower is better). The results show that MuDo-NLL outperforms MuDo-Cov in nearly every matched condition, with the largest advantage occurring in bidirected-edge recovery. Both objectives benefit from additional domains and larger sample sizes, but MuDo-NLL shows a clear advantage.
\begin{table}[h]
\caption{Multi-domain Scalings results: average F1 and SHD $\pm$ standard deviation over 20 random seeds. The varying parameter and the best method for each metric in each row are shown in bold.}
\label{tab:scaling-f1-shd}
\centering
\begin{small}
\begin{tabular}{llcccc}
\toprule
Setting & Value & \multicolumn{2}{c}{F1} & \multicolumn{2}{c}{SHD} \\
\cmidrule(lr){3-4}\cmidrule(lr){5-6}
 &  & MuDo-NLL & MuDo-Cov & MuDo-NLL & MuDo-Cov \\
\midrule
\multicolumn{6}{l}{\textit{Directed edges}} \\
\multicolumn{6}{l}{\emph{Dimension}} \\
$n$ & \textbf{5} & \textbf{0.927 $\pm$ 0.125} & 0.736 $\pm$ 0.182 & \textbf{1.300 $\pm$ 2.227} & 4.650 $\pm$ 2.925 \\
$n$ & \textbf{10} & \textbf{0.883 $\pm$ 0.196} & 0.821 $\pm$ 0.076 & \textbf{4.500 $\pm$ 5.708} & 7.900 $\pm$ 3.227 \\
$n$ & \textbf{20} & \textbf{0.854 $\pm$ 0.280} & 0.676 $\pm$ 0.261 & \textbf{8.750 $\pm$ 12.586} & 27.400 $\pm$ 15.091 \\
\addlinespace
\multicolumn{6}{l}{\emph{Number of domains}} \\
$T/n$ & \textbf{1} & \textbf{0.883 $\pm$ 0.196} & 0.821 $\pm$ 0.076 & \textbf{4.500 $\pm$ 5.708} & 7.900 $\pm$ 3.227 \\
$T/n$ & \textbf{2} & \textbf{0.911 $\pm$ 0.171} & 0.805 $\pm$ 0.219 & \textbf{3.550 $\pm$ 6.419} & 7.350 $\pm$ 5.373 \\
$T/n$ & \textbf{3} & \textbf{0.976 $\pm$ 0.087} & 0.880 $\pm$ 0.130 & \textbf{1.100 $\pm$ 4.012} & 5.350 $\pm$ 5.842 \\
\addlinespace
\multicolumn{6}{l}{\emph{Sample size}} \\
$N$ & \textbf{1000} & \textbf{0.843 $\pm$ 0.101} & 0.738 $\pm$ 0.126 & \textbf{7.200 $\pm$ 4.360} & 11.700 $\pm$ 5.151 \\
$N$ & \textbf{5000} & \textbf{0.883 $\pm$ 0.196} & 0.821 $\pm$ 0.076 & \textbf{4.500 $\pm$ 5.708} & 7.900 $\pm$ 3.227 \\
$N$ & \textbf{50000} & \textbf{0.967 $\pm$ 0.108} & 0.837 $\pm$ 0.064 & \textbf{1.300 $\pm$ 4.207} & 7.450 $\pm$ 3.086 \\
\addlinespace
\multicolumn{6}{l}{\emph{Graph density}} \\
$k_X$ & \textbf{1.0} & \textbf{0.773 $\pm$ 0.111} & 0.618 $\pm$ 0.089 & \textbf{6.050 $\pm$ 3.561} & 11.400 $\pm$ 3.719 \\
$k_X$ & \textbf{2.0} & \textbf{0.883 $\pm$ 0.196} & 0.821 $\pm$ 0.076 & \textbf{4.500 $\pm$ 5.708} & 7.900 $\pm$ 3.227 \\
$k_X$ & \textbf{3.0} & \textbf{0.818 $\pm$ 0.268} & 0.761 $\pm$ 0.204 & \textbf{9.000 $\pm$ 12.490} & 13.500 $\pm$ 11.062 \\
\addlinespace
\multicolumn{6}{l}{\emph{Confounding density}} \\
$k_U$ & \textbf{0.5} & \textbf{0.883 $\pm$ 0.196} & 0.821 $\pm$ 0.076 & \textbf{4.500 $\pm$ 5.708} & 7.900 $\pm$ 3.227 \\
$k_U$ & \textbf{1.0} & \textbf{0.803 $\pm$ 0.277} & 0.793 $\pm$ 0.166 & \textbf{7.350 $\pm$ 9.051} & 7.950 $\pm$ 4.161 \\
\addlinespace
\multicolumn{6}{l}{\emph{Noise distribution}} \\
Noise & \textbf{gauss} & \textbf{0.883 $\pm$ 0.196} & 0.821 $\pm$ 0.076 & \textbf{4.500 $\pm$ 5.708} & 7.900 $\pm$ 3.227 \\
Noise & \textbf{exp} & \textbf{0.850 $\pm$ 0.176} & 0.816 $\pm$ 0.070 & \textbf{6.050 $\pm$ 5.031} & 8.150 $\pm$ 3.133 \\
Noise & \textbf{gumbel} & 0.832 $\pm$ 0.218 & \textbf{0.837 $\pm$ 0.065} & \textbf{6.650 $\pm$ 6.556} & 7.600 $\pm$ 3.283 \\
\midrule
\multicolumn{6}{l}{\textit{Bidirected edges}} \\
\multicolumn{6}{l}{\emph{Dimension}} \\
$n$ & \textbf{5} & \textbf{0.576 $\pm$ 0.142} & 0.419 $\pm$ 0.088 & \textbf{3.650 $\pm$ 1.496} & 6.700 $\pm$ 1.174 \\
$n$ & \textbf{10} & \textbf{0.783 $\pm$ 0.142} & 0.444 $\pm$ 0.107 & \textbf{4.500 $\pm$ 5.615} & 17.900 $\pm$ 5.418 \\
$n$ & \textbf{20} & \textbf{0.857 $\pm$ 0.156} & 0.322 $\pm$ 0.099 & \textbf{5.550 $\pm$ 7.884} & 70.050 $\pm$ 46.479 \\
\addlinespace
\multicolumn{6}{l}{\emph{Number of domains}} \\
$T/n$ & \textbf{1} & \textbf{0.783 $\pm$ 0.142} & 0.444 $\pm$ 0.107 & \textbf{4.500 $\pm$ 5.615} & 17.900 $\pm$ 5.418 \\
$T/n$ & \textbf{2} & \textbf{0.864 $\pm$ 0.225} & 0.556 $\pm$ 0.208 & \textbf{3.000 $\pm$ 6.224} & 12.650 $\pm$ 8.940 \\
$T/n$ & \textbf{3} & \textbf{0.933 $\pm$ 0.156} & 0.707 $\pm$ 0.184 & \textbf{1.600 $\pm$ 4.893} & 6.550 $\pm$ 5.463 \\
\addlinespace
\multicolumn{6}{l}{\emph{Sample size}} \\
$N$ & \textbf{1000} & \textbf{0.465 $\pm$ 0.092} & 0.305 $\pm$ 0.057 & \textbf{15.250 $\pm$ 4.241} & 29.150 $\pm$ 3.360 \\
$N$ & \textbf{5000} & \textbf{0.783 $\pm$ 0.142} & 0.444 $\pm$ 0.107 & \textbf{4.500 $\pm$ 5.615} & 17.900 $\pm$ 5.418 \\
$N$ & \textbf{50000} & \textbf{0.910 $\pm$ 0.185} & 0.496 $\pm$ 0.143 & \textbf{2.300 $\pm$ 7.035} & 14.250 $\pm$ 5.609 \\
\addlinespace
\multicolumn{6}{l}{\emph{Graph density}} \\
$k_X$ & \textbf{1.0} & \textbf{0.777 $\pm$ 0.165} & 0.451 $\pm$ 0.107 & \textbf{4.450 $\pm$ 5.558} & 17.200 $\pm$ 5.672 \\
$k_X$ & \textbf{2.0} & \textbf{0.783 $\pm$ 0.142} & 0.444 $\pm$ 0.107 & \textbf{4.500 $\pm$ 5.615} & 17.900 $\pm$ 5.418 \\
$k_X$ & \textbf{3.0} & \textbf{0.658 $\pm$ 0.272} & 0.369 $\pm$ 0.098 & \textbf{9.750 $\pm$ 10.543} & 23.650 $\pm$ 7.184 \\
\addlinespace
\multicolumn{6}{l}{\emph{Confounding density}} \\
$k_U$ & \textbf{0.5} & \textbf{0.783 $\pm$ 0.142} & 0.444 $\pm$ 0.107 & \textbf{4.500 $\pm$ 5.615} & 17.900 $\pm$ 5.418 \\
$k_U$ & \textbf{1.0} & \textbf{0.823 $\pm$ 0.125} & 0.704 $\pm$ 0.116 & \textbf{8.250 $\pm$ 6.766} & 15.100 $\pm$ 5.803 \\
\addlinespace
\multicolumn{6}{l}{\emph{Noise distribution}} \\
Noise & \textbf{gauss} & \textbf{0.783 $\pm$ 0.142} & 0.444 $\pm$ 0.107 & \textbf{4.500 $\pm$ 5.615} & 17.900 $\pm$ 5.418 \\
Noise & \textbf{exp} & \textbf{0.764 $\pm$ 0.166} & 0.399 $\pm$ 0.082 & \textbf{4.700 $\pm$ 4.438} & 19.950 $\pm$ 4.395 \\
Noise & \textbf{gumbel} & \textbf{0.714 $\pm$ 0.206} & 0.476 $\pm$ 0.099 & \textbf{6.100 $\pm$ 6.095} & 15.350 $\pm$ 5.324 \\
\bottomrule
\end{tabular}
\end{small}
\end{table}

\FloatBarrier

\subsection{Comparing with FCI}
\label{app: fci}

We compare \texttt{LinADMG-MuDoNLL} , \texttt{LinADMG-diagGMM} , \texttt{LinADMG-UMNI} , \texttt{LinADMG-SP} , and FCI on four data-generating processes: DiagGMMs, UMN-Hard Interventions, UP-Observations, and MD-Scalings. In every experiment, we set $n=10, N=5000, T=1, k_X=2.0, k_U=0.5$ and Gaussian noise. Results are averaged over 20 random seeds. Because FCI returns a partial ancestral graph (PAG), the ADMG estimated by each LinADMG method, as well as the ground-truth ADMG, is converted to its corresponding PAG, using a maximal ancestral projection where necessary. All methods are then evaluated in the common PAG representation using skeleton F1, arrowhead F1, and tail F1. Each bar reports the mean over the 20 seeds, while the error bar represents one standard deviation.

\texttt{LinADMG-diagGMM}  provides the strongest overall performance, particularly for the DiagGMMs and UP-Observations DGPs, where it accurately recovers both the skeleton and the endpoint marks. \texttt{LinADMG-MuDoNLL} performs best in the MD-Scalings DGP. Under UMNI-Hard Interventions DGP, FCI is competitive in skeleton recovery, but the LinADMG methods generally recover arrowheads and tails more accurately. \texttt{LinADMG-SP}  often recovers a reasonable skeleton but is less successful in finding endpoint orientation, while \texttt{LinADMG-UMNI}  shows greater variability, especially for tail recovery. 
\begin{figure}[h]
    \centering
    \includegraphics[width=0.9\textwidth]{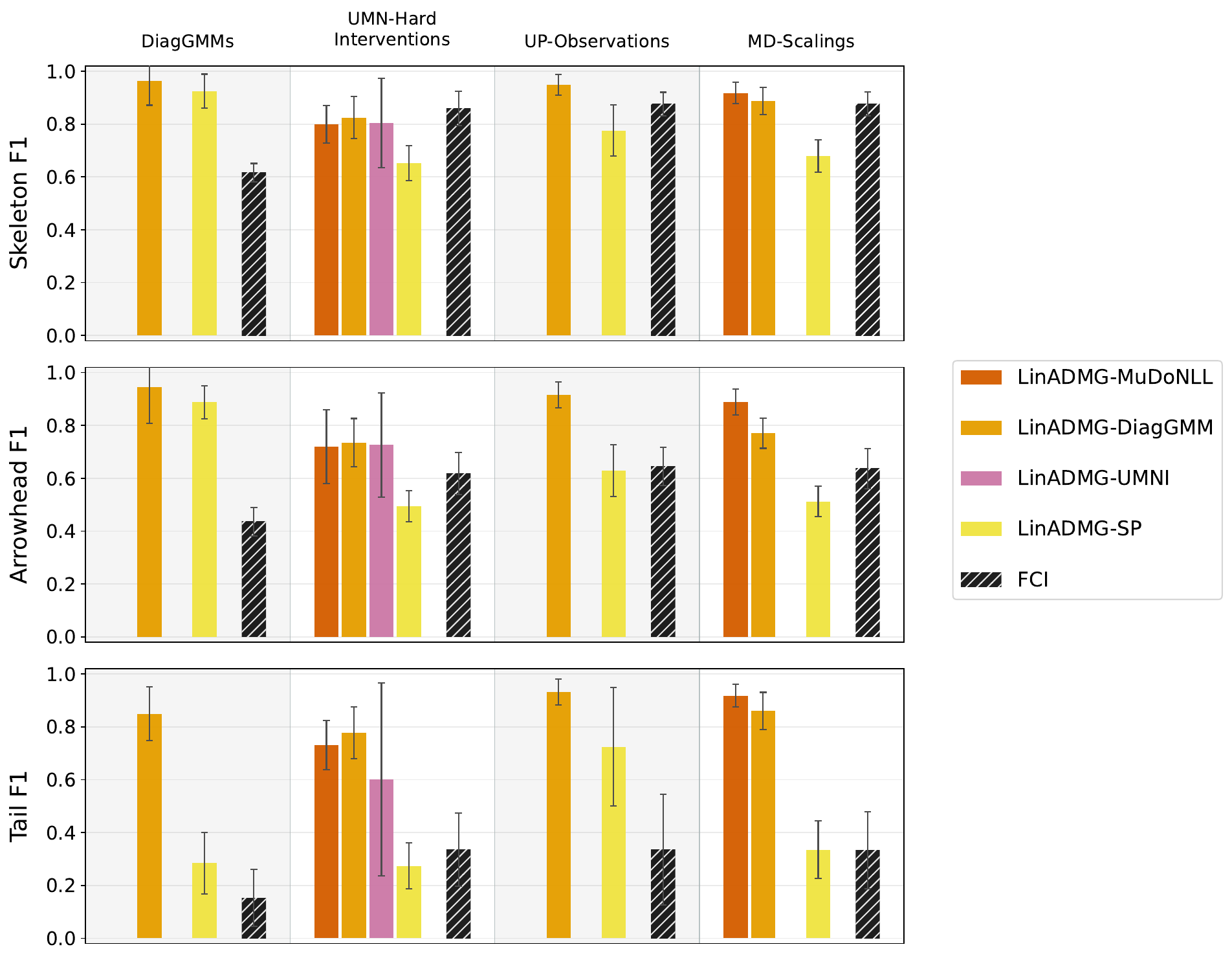}
    \caption{
       PAG recovery performance of \texttt{LinADMG-MuDoNLL} , \texttt{LinADMG-diagGMM} , \texttt{LinADMG-UMNI} , \texttt{LinADMG-SP} and FCI across four data-generating processes. The estimated and ground-truth ADMGs are converted to PAGs before evaluation, while FCI is evaluated using its native PAG output. The three panels report skeleton, arrowhead, and tail F1 scores. Bars show the mean and error bars indicate one standard deviation over 20 seeds. All experiments use $n=10, N=5000,T=n, k_X=2, k_U=0.5$.
    }
    \label{fig:FCI}
\end{figure}

\FloatBarrier

\subsection{Bow-free ADMG with exponential exogenous variables}
\label{app:bowfree}
\begin{figure}[t]
    \centering
    \includegraphics[width=1.0\textwidth]{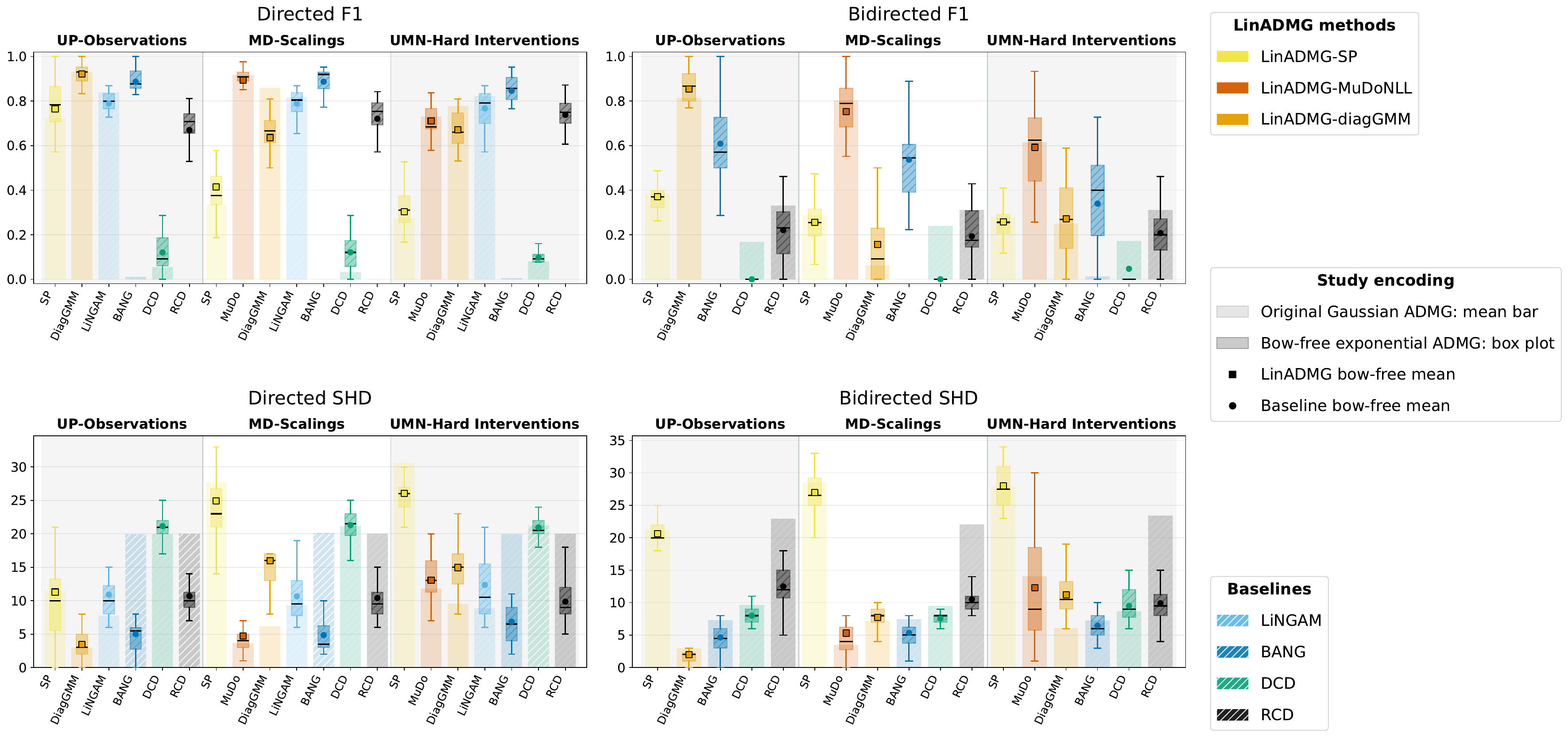}
    \caption{
        Comparison of ADMG estimation performance across the 20 method-DGP configurations shared by the original Gaussian-ADMG study and the bow-free exponential-noise study (20 seeds per configuration). Translucent bars show the mean for the original study as in Fig.\ref{fig:admg_f1_shd}, while box plots show the bow-free study. Filled markers denote bow-free means: squares for LinADMG- methods and circles for baselines. LiNGAM is omitted from bidirected panels because it does not estimate bidirected edges.
    }
    \label{fig:bowfree}
\end{figure}

LinADMG- methods are comparatively stable on all three data-generating processes, with \texttt{LinADMG-diagGMM}  remaining strongest in general. As expected, the most pronounced changes occur for BANG and RCD as they required for non-Gaussian and bow-free ADMG: BANG’s directed F1 rises from approximately zero to 0.886-0.846 across the three DGPs, while RCD rises from zero to 0.670-0.738. BANG also gains substantial bidirected accuracy. DCD remains weak for directed recovery and selects almost no bidirected edges in the bow-free UP-Observations and MD-Scalings settings.

\end{document}